\documentclass{article}

\usepackage{microtype}
\usepackage{graphicx}
\graphicspath{{../}} 
\usepackage{subcaption}
\usepackage{booktabs} 

\usepackage{hyperref}

\usepackage[preprint]{icml2026}

\usepackage{amsmath}
\usepackage{amssymb}
\usepackage{mathtools}
\usepackage{amsthm}
\usepackage{enumitem}
\usepackage{multirow}

\usepackage[capitalize,noabbrev]{cleveref}

\theoremstyle{plain}
\newtheorem{theorem}{Theorem}[section]
\newtheorem{proposition}[theorem]{Proposition}
\newtheorem{lemma}[theorem]{Lemma}
\newtheorem{corollary}[theorem]{Corollary}
\theoremstyle{definition}

\newtheorem{assumption}[theorem]{Assumption}
\crefname{assumption}{Assumption}{Assumptions}
\theoremstyle{remark}
\newtheorem{remark}[theorem]{Remark}
\crefname{remark}{Remark}{Remarks}

\usepackage[textsize=tiny]{todonotes}

\icmltitlerunning{Canary}

\begin{document}

\twocolumn[
  \icmltitle{Cantelli Constrained Policy Optimization}



  \icmlsetsymbol{equal}{*}

  \begin{icmlauthorlist}
    \icmlauthor{Rohan Tangri}{ox}
    \icmlauthor{Jan-Peter Calliess}{ox}
  \end{icmlauthorlist}

  \icmlaffiliation{ox}{Machine Learning Research Group, University of Oxford, Oxford, United Kingdom}

  \icmlcorrespondingauthor{Rohan Tangri}{rohan.tangri@reub.ox.ac.uk}
  \icmlcorrespondingauthor{Jan-Peter Calliess}{jan@robots.ox.ac.uk}

  \icmlkeywords{Machine Learning, ICML}

  \vskip 0.3in ]



\printAffiliationsAndNotice{}  

\begin{abstract}
We introduce \emph{Canary}, a risk-averse method designed to optimize
Value-at-Risk (VaR) constrained reinforcement learning (RL) problems. We employ
Cantelli's inequality to obtain a tractable, conservative and smooth bound on
the VaR constraint based on the first two moments of the cost return. This
yields a constraint estimator that remains stable with tight violation
thresholds in dense cost regimes. Extending the trust-region framework of the
Constrained Policy Optimization (CPO) method, we further provide worst-case
bounds for both policy improvement and constraint violation during the training
process. Empirically during training, Canary is the only method that reliably satisfies
the VaR constraint in every environment tested.
\end{abstract}

\section{Introduction}
\label{intro}

\emph{Reinforcement Learning (RL)} has demonstrated remarkable utility in
optimizing complex decision-making simulations such as autonomous driving
\cite{rldriving-survey}, robotics \cite{gae, ppo, trpo} and finance
\cite{rlfin-survey, rl-hft}. The standard RL objective typically seeks to
maximize the expected return. While effective in simulations, this formulation
is often insufficient for real-life scenarios, especially when bounding the probability
of a costly catastrophic failure is a prerequisite for deployment. In these
settings, designing a risk-aware reward function via a soft constraint can be impractical. The cost of tuning may be too high, while choosing a setting a priori might lead to intractable learning or unsafe policies. This motivates the use of hard constraints on tail risks, which are often a more natural formulation to facilitate safety.

A natural hard constraint to achieve safety is a \emph{Value-at-Risk
(VaR)} constraint, where the objective is to \emph{first} satisfy a bound on the
probability that the cumulative cost exceeds a limit, and only then to maximize
reward \cite{var-cheb, tagawa}. Here, cost is a distinct penalty for
\emph{unsafe} states that might be distinct from the reward signal of the RL
agent's objective. While the term VaR is chiefly used in finance, the way we employ
it is general and extends naturally to probabilistic safety constraints in
control and robotics. For example, in finance we may wish to avoid a margin
call, while in robotics we would want to avoid actions that destroy physical
components.

Existing methods struggle precisely in the practically important regime of
strict (small) violation thresholds. Indicator based methods
\cite{geibel-wysotzki, var-lagrange} suffer from sample inefficiency through
cost sparsity \cite{her} and instability
as the violation probability shrinks (\cref{sec:relative-variance}). In finance, \emph{Conditional Value-at-Risk (CVaR)} is often favored
as a tractable conservative bound on VaR \cite{cvar1, cvar2}. However, empirical
CVaR estimators are also unstable in tight VaR violation threshold regimes (\cref{sec:relative-variance}).

We introduce \emph{Canary},\footnote{Like a canary in a coal mine, the Cantelli
constraint senses the approaching probabilistic tail risk boundary and alerts the
policy.} a novel approach with primary contributions:
\begin{itemize}[nosep]
\item We formulate a smooth and tractable safe approximation \citep{nemirovski2012safe} of VaR constraints in RL using the Cantelli inequality: the tightest possible distribution-free bound relying only on the first two moments of the cost return distribution \cite{bertsimas2005}  (\cref{sec:method}).
\item We provide a worst-case analysis of the constraint violation
during training, extending the guarantees of the trust-region based CPO method
\cite{cpo} to the VaR setting (\cref{sec:method}).
\item We characterize the relative variance of VaR constraint
estimators, showing indicator and CVaR estimators collapse in low-violation
regimes while Cantelli remains stable (\cref{sec:relative-variance}).
\item We empirically demonstrate that Canary settles at the deepest safety level, achieves the least converged violation-probability jitter, and maintains safety most reliably on all three benchmark environments (\cref{sec:results}).
\end{itemize}

\section{Related Work}

\paragraph{Trust region methods.}
Safe RL typically constrains expected cost returns using Lagrangian methods
\cite{lagrange-constrained}, which offer no bounds on violation during training.
To resolve this, Constrained Policy Optimization (CPO) \cite{cpo} employs
trust-region theory \cite{trpo} to guarantee worst-case limits on training-time
violations. Canary directly inherits these rigorous trust-region guarantees, but
fundamentally extends them from average-case expected costs to worst-case tail
risk.

\paragraph{Value-at-Risk (VaR) optimization.}
Constraining the probability of violation
casts the problem as a chance-constrained MDP, where dynamic programming or search bounds the violation
probability under a known model \cite{ono-ccdp, santana-rao}. These
methods do not extend to the model-free, continuous-control setting we
address. Within RL, a Bernoulli indicator which triggers upon violation is often used to directly
estimate the VaR bound, introducing an extremely sparse, sample-inefficient cost
signal \cite{geibel-wysotzki, her}. Alternatively, \citet{var-lagrange} propose
Lagrangian policy gradients for VaR constraints, but lack training-time safety
guarantees. Canary resolves these issues: applying standard state
augmentation \citep{xu-mannor, saute} to the second cost moment gives dense Markovian cost assignment, and our Cantelli formulation provides a smooth, differentiable safe approximation compatible with worst-case trust-region bounds.

\paragraph{Conditional Value-at-Risk (CVaR).}
Most existing RL literature on tail risk defaults to CVaR constraints
\cite{cvar-cpo, cvar-ppo}, or optimizes CVaR directly as an objective
\cite{distrl}. These approaches frequently leverage
distributional RL to compute quantile-based boundary estimates \cite{cvar-cpo,
distrl}. This introduces architectural complexity and relies on empirical tail
estimators that yield fragile scaling of constraint gradients derived strictly
from a worst fraction of trajectories. Canary bypasses this bottleneck; by
projecting the boundary via the first two moments of the full batch, our method
provides stable gradients that proactively constrain the tail. WCSAC
\cite{wcsac} also reasons about tail risk through moments, but does so by
assuming a Gaussian cost distribution within an off-policy framework. We exclude
it from our on-policy comparison, and note that the Gaussian tail
underestimates risk in our
experiments (\cref{rem:gaussian}).

\section{Preliminaries}
\label{sec:preliminaries}

\subsection{Constrained Markov Decision Process}
\label{sec:CMDP}

We model an agent's interaction with an environment as a \emph{Constrained
Markov Decision Process (CMDP)} \cite{cmdp}. Let $\Delta(\mathcal{X})$ represent
a set of probability distributions over a set $\mathcal{X}$. A CMDP is
defined by the tuple $(\mathcal{S}, \mathcal{A}, T, \mathcal{R}, \mathcal{C},
d_0)$. Here, $\mathcal{S}$ is the state space, $\mathcal{A}$ is the action space, $T: \mathcal{S} \times
\mathcal{A} \rightarrow \Delta(\mathcal{S})$ is the state transition kernel, $\mathcal{R}: \mathcal{S} \times \mathcal{A} \rightarrow
\Delta(\mathbb{R})$ is the reward kernel, $\mathcal{C}:
\mathcal{S} \times \mathcal{A} \rightarrow \Delta(\mathbb{R})$ is the cost kernel, and $d_0 \in \Delta(\mathcal{S})$ is the initial state distribution. The reward kernel measures task performance, while the cost kernel measures
safety. A \emph{policy} $\pi: \mathcal{S} \rightarrow \Delta(\mathcal{A})$
maps states to a probability distribution over actions. We assume the agent
interacts with the environment to generate infinite length \emph{trajectories} $\tau =
(s_t, a_t, r_t, c_t)_{t \in \mathcal{N}}$, which are sampled at time $t$ with \emph{actions} $a_t \sim \pi(\cdot|s_t)$, \emph{states} $s_{t+1} \sim T(\cdot|s_t, a_t)$, \emph{rewards} $r_t \sim \mathcal{R}(\cdot|s_t, a_t)$ and \emph{costs} $c_t \sim \mathcal{C}(\cdot|s_t, a_t)$, where typically in theoretical analysis $\mathcal{N} = \mathbb{N}_0$. The standard
reinforcement learning objective is to maximize the
\emph{expected reward return}
\begin{equation}
    \label{eq:reward-return}
    J(\pi) = \mathbb{E}_{\tau \sim \pi} \Big[ \sum_{t \in \mathcal N} \gamma^t r_t \Big],
\end{equation}
where
$\gamma \in [0, 1)$ is the \emph{reward discount factor}.

In a safety-critical setting, we are also concerned with the
\emph{cost return} $C(\tau) = \sum_{t \in \mathcal N} \gamma_c^t c_t$ and its expectation
\begin{gather}
    \label{eq:cost-return}
    \mu(\pi) = \mathbb{E}_{\tau \sim \pi}[C(\tau)] = \mathbb{E}_{\tau \sim \pi}\Big[\sum_{t \in \mathcal N} \gamma_c^t c_t\Big].
\end{gather}
Here $\gamma_c \in[0, 1)$ is the \emph{cost discount factor}.

The standard CMDP objective maximizes $J(\pi)$ subject to a \emph{cost limit} $\ell \in \mathbb R$ on $\mu(\pi)$, the Optimization Problem (\emph{OP}):
\begin{equation}
\begin{aligned}
    \pi^* = \arg \max_\pi & \text{ } J(\pi) \\
    \text{s.t.} \text{ } & \mu(\pi) \le \ell.
    \label{eq:standardconstr_CMDP}
\end{aligned}
\end{equation}

\subsection{Constrained Policy Optimization}
\label{sec:CPO}
Constrained Policy Optimization (CPO) is an iterative algorithm for solving
CMDPs that bounds worst-case constraint violation and performance degradation magnitudes at
each update step \cite{cpo}. To ensure stable learning, CPO employs a
trust-region approach \cite{trpo}, constraining the step size of the policy
update via the Kullback-Leibler (KL) divergence. First, let $d_\pi(s) = (1 - \gamma) \sum_{t \in \mathcal{N}}\gamma^t \mathbb{P}(s_t=s\mid\pi)$ and
$d_\pi^c(s) = (1 - \gamma_c) \sum_{t \in \mathcal{N}}\gamma_c^t \mathbb{P}(s_t=s\mid\pi)$ define the reward and cost \emph{discounted state visitation frequencies} respectively for policy $\pi$.

In our iterative framework, we aim to update the policy $\pi_k$ at each update
step $k$ to a new policy $\pi_{k+1}$ satisfying OP~\ref{eq:standardconstr_CMDP}. Given a candidate policy $\pi \ne \pi_k$, it is
impossible to calculate $J(\pi)$ or $\mu(\pi)$ for $\pi$ 
given trajectories only sampled from $\pi_k$. Instead, \emph{trust-region surrogates} $L(\pi)$ and $L_{\mu}(\pi)$ for
$J(\pi)$ and $\mu(\pi)$ respectively, which purely sample from policy $\pi_k$, can be calculated:
\begin{gather}
    \label{eq:return-surrogate}
    L(\pi) = J(\pi_k) + \frac{1}{1 - \gamma} \mathbb{E}_{\substack{s \sim d_{\pi_k} \\ a \sim \pi}}\left[ A_{\pi_k}(s,a) \right] \\
    \label{eq:lmu}
    L_{\mu}(\pi) = \mu(\pi_k) + \frac{1}{1 - \gamma_c} \mathbb{E}_{\substack{s \sim d^c_{\pi_k} \\ a \sim \pi}}\left[ A^C_{\pi_k}(s,a) \right].
\end{gather}
Here $A_{\pi_k}$ and $A_{\pi_k}^{C}$ are the \emph{advantage functions} \cite{gae}
following policy $\pi_k$ for the reward and cost, respectively.

Crucially, $L(\pi_k) = J(\pi_k)$ and $\nabla
L(\pi_k) =~\nabla J(\pi_k)$. Therefore, although not necessarilly linear functions, these surrogates are often
referred to as valid first-order approximations \cite{trpo, cpo}. Consequently, at each iteration $k$, given a policy $\pi_k$,
CPO seeks a new policy $\pi_{k+1}$ that solves the following local optimization
problem:

\begin{equation}
\label{eq:cpo}
\begin{aligned}
    \pi_{k+1} = \arg \max_{\pi} & \text{ } L(\pi) \\
    \text{s.t.} \text{ } & L_{\mu}(\pi) \le \ell \\
    & \bar{D}_{KL}(\pi, \pi_k) \le \delta
\end{aligned}
\end{equation}%
where $\bar{D}_{KL}(\pi, \pi_k) = \mathbb{E}_{s\sim d_{\pi_k}^c}[D_{KL}(\pi(a\mid s), \pi_k(a\mid s))]$ is the expected KL divergence between the policies $\pi_k$ and $\pi$, which by
setting hyperparameter radius $\delta > 0$ defines a ball in the policy space called the
\emph{trust region} \cite{trpo}.

\subsection{Value-at-Risk Objective}
\label{sec:VaR}
The \emph{Value-at-Risk (VaR)} of the cost return at a \emph{violation
threshold} $\varepsilon \in (0, 1)$ is
\begin{equation}
  \mathrm{VaR}_{\varepsilon}(\pi)
  = \inf\big\{\, \eta \in \mathbb{R} : \mathbb{P}\big(C(\tau) > \eta\big)
  \le \varepsilon \,\big\}.
  \label{eq:strict-var}
\end{equation}
We write $\eta^* := \mathrm{VaR}_\varepsilon(\pi)$; by right-continuity of $\eta \mapsto \mathbb P(C(\tau) > \eta)$, $\mathbb P(C(\tau) > \eta^*) \le \varepsilon$.
In contrast to the standard constraint on expected cost return in the CMDP
objective (\ref{eq:standardconstr_CMDP}), we are interested in bounding the tail
risk: given the cost limit $\ell$ (\cref{sec:CMDP}), we impose the
probabilistic constraint $\mathbb{P}(C(\tau) > \ell) \le
\varepsilon$, which is equivalent to $\mathrm{VaR}_{\varepsilon}(\pi) \le \ell$. We refer to this as the \emph{VaR constraint}. The
resulting objective becomes
\begin{equation}
\label{eq:var-bound}
\boxed{
\begin{aligned}
    \pi^* = \arg \max_\pi & \text{ } J(\pi) \\
    \text{s.t.} \text{ } & \mathbb{P}(C(\tau) > \ell) \le \varepsilon.
\end{aligned}
}
\end{equation}
Note that we have not specified the policy space being maximized over.
Typically, this will be given by some parametric class of policies, reducing the
objective to a constrained parameter OP.

\subsection{Indicator Based Methods}
\label{sec:indicator}
In principle, OP
\ref{eq:var-bound} could be solved as a special case of the standard CMDP
problem in OP \ref{eq:standardconstr_CMDP} by introducing an indicator-based cost return $I(\tau)$
\cite{geibel-wysotzki, var-lagrange}:
\begin{equation}
    I(\tau) = \mathbf{1}(C(\tau) > \ell) =
    \begin{cases}
    1 & \text{if } \sum_{t \in \mathcal N} \gamma_c^t c_t > \ell \\
    0 & \text{otherwise},
    \end{cases}
\end{equation}

whose expectation is the violation probability; with $I(\tau)$ as the cost return, the constraint in OP~\ref{eq:standardconstr_CMDP} becomes:
\begin{equation}
    \label{eq:indicator}
    \mu_I(\pi) = \mathbb{E}_{\tau\sim \pi}[I(\tau)] = \mathbb{P}(C(\tau) > \ell) \le \varepsilon.
\end{equation}
We will refer to this as the \emph{indicator constraint}.

\subsection{Conditional Value-at-Risk Methods}
\label{sec:cvar-fragility}

An alternative for the VaR objective is to constrain the Conditional
Value-at-Risk (CVaR). This is a \emph{safe approximation} of the VaR constraint \citep{nemirovski2012safe}:
any policy satisfying the CVaR constraint at target violation threshold
$\varepsilon$ also satisfies the VaR constraint. CVaR can be tractably formulated using the
Rockafellar-Uryasev dual representation \cite{cvar1, cvar2}:
\begin{equation}
    \label{eq:rockafellar}
    \text{CVaR}_\varepsilon(\pi) = \min_\eta \left( \eta + \frac{1}{\varepsilon} \mathbb{E}_{\tau} \left[ (C(\tau) - \eta)^+ \right] \right)
\end{equation}
where $(x)^+ = \max(0,x)$; the $(1-\varepsilon)$-quantile $\eta^*$ of \cref{eq:strict-var} is a minimizer. The CVaR constrained OP is then:
\begin{equation}
    \label{eq:cvar-opt}
    \begin{aligned}
    \pi^* = \arg \max_\pi & \text{ } J(\pi) \\
    \text{s.t.} \text{ } & \text{CVaR}_\varepsilon(\pi) \le \ell.
\end{aligned}
\end{equation}

\subsection{Cantelli's Inequality}
\label{sec:cantelli-stability}

Cantelli's inequality (also known as the one-sided Chebyshev inequality)
provides a tighter version of the Chebyshev inequality for one-sided tail
bounds. Given the random cost return $C(\tau)$ with finite mean $\mu(\pi)$
(\cref{eq:cost-return}) and variance $\sigma^2(\pi) =
\text{Var}_{\tau\sim\pi}[C(\tau)]$, Cantelli's inequality states that for any
$\lambda > 0$:
\begin{equation}
    \label{eq:cantelli}
    \mathbb{P}(C(\tau) - \mu(\pi) \ge \lambda) \le \frac{\sigma^2(\pi)}{\sigma^2(\pi) + \lambda^2}.
\end{equation}
This holds for any distribution with finite first and second moments, which makes it more conservative than a bound that exploits further knowledge of the distribution of $C(\tau)$. Crucially, however, no distribution-free bound using only the first two moments can be tighter on $\mathbb{R}$ \cite{bertsimas2005}.

\section{Method}
\label{sec:method}

In this section, we present Canary, a risk-averse, stable (\cref{sec:relative-variance}) algorithm for
enforcing Value-at-Risk constraints in an online RL framework using Cantelli's inequality. Further proofs and derivations can be found in Appendix \ref{app:proofs}.

\subsection{Cantelli Approximated Value-at-Risk}
\label{sec:cantelli-var}

We first employ Cantelli's inequality (\ref{eq:cantelli}) to construct a smooth
safe approximation of the original VaR constraint in OP \ref{eq:var-bound}. Substituting $\lambda = \ell - \mu(\pi)$ into (\ref{eq:cantelli}) can be shown to be equivalent to the \emph{Cantelli constraint} (refer to \cref{app:cantelli-var}):
\begin{equation}\label{eq:chebyshev_constraint}
    B(\pi) := \left(\frac{1}{\varepsilon} - 1\right) \sigma^2(\pi) - [\ell - \mu(\pi)]^2 \le 0.
\end{equation}
Cantelli's inequality requires $\lambda > 0$, so $\mu(\pi) < \ell$. However, the boundary case $\mu(\pi) = \ell$ also satisfies the VaR constraint under $B(\pi) \le 0$ (\cref{app:cantelli-var}). The non-strict \emph{mean constraint} $\mu(\pi) \le \ell$ can therefore be imposed, and OP~\ref{eq:var-bound} is transformed into the \emph{Cantelli-VaR} form:
\begin{equation}
\label{eq:cantelli-objective}
\begin{aligned}
    \pi^* = \arg \max_\pi & \text{ } J(\pi) \\
    \text{s.t.} \text{ } & B(\pi) \le 0,\quad \mu(\pi) \le \ell.
\end{aligned}
\end{equation}
We call $(\ell, \varepsilon)$ a \emph{feasible setting} if some policy $\pi$ satisfies OP \ref{eq:cantelli-objective}. The Cantelli constraint is equivalent to a worst-case VaR bound over all cost distributions sharing the first two moments of $C(\tau)$ \citep{elghaoui2003}, which equals the worst-case CVaR over that family \citep{zymler2013joint}. Since CVaR bounds VaR \citep{cvar1, cvar2}, the Cantelli-feasible set of policies for OP \ref{eq:cantelli-objective} is a subset of the CVaR-feasible set of OP \ref{eq:cvar-opt}, which is in turn a subset of the VaR-feasible set of OP \ref{eq:var-bound}. As such, any policy satisfying the Cantelli-VaR constraint is guaranteed to satisfy the original VaR constraint, but not vice versa. The Cantelli-VaR constraint is thus also a safe approximation of the VaR constraint. This conservatism means a policy that appears feasible only through estimation error can still satisfy the VaR constraint, at the price of a smaller feasible set of policies; that price is quantified in \cref{sec:anatomy}.

\subsection{Augmented Cost Formulation}

We wish to facilitate the derivation of worst-case training time constraint violations for the Cantelli-VaR constraint (\ref{eq:chebyshev_constraint}) in lieu of those provided by CPO. This requires the cost variance $\sigma^2(\pi) = \mathbb{E}[C(\tau)^2] - \mathbb{E}[C(\tau)]^2$ and expectation $\mu(\pi)$ to be estimable using trust-region surrogates. In turn, these make use of advantage functions, which have a form given by a discounted sum of local per-step terms. While $\mu(\pi)$ is already of this form, the second moment $\mathbb{E}[C(\tau)^2]$ is not; we therefore rewrite it as a discounted sum of local per-step terms.

Let $y_t$ be the discounted accumulated cost up to time $t$ with
\begin{equation}
    y_{t+1} = y_t + \gamma_c^t c_t, \quad y_0 = 0. \label{eq:y-recursion}
\end{equation}
The square of the cost return can be decomposed as follows (\cref{app:square-cost}):

\begin{equation}
    \label{eq:square-cost}
    C(\tau)^2 = \sum_{t \in \mathcal N} \gamma_c^t (\gamma_c^t c_t^2 + 2 y_t c_t).
\end{equation}

Therefore, given $\beta = \frac{1}{\varepsilon}-1$, the cost second moment can be written as the expected discounted return of the \emph{augmented cost} $\tilde{c}_t$:
\begin{gather}
    \label{eq:aug-cost}
    \tilde{c}_t = \gamma_c^{t}c_t^2 + 2 y_t c_t \\
    \label{eq:aug-cost-return}
    \tilde{\mu}(\pi) = \mathbb{E}_{\tau \sim \pi} \Big[\sum_{t=0}^\infty \gamma_c^t \tilde{c}_t \Big] = \mathbb{E}_{\tau \sim \pi}[C(\tau)^2].
\end{gather}

However, this augmented cost is not Markovian; that is, the state-action tuple $(s_t, a_t)$ alone is insufficient to
determine the distribution of $\tilde{c}_t$. Akin to \citep{xu-mannor, saute}, we augment the state space with $y_t$
and $\gamma_c^t$, to give the augmented state $x_t$:
\begin{equation}
x_t = (s_t, y_t, \gamma_c^t). \label{eq:aug-state}
\end{equation}

Letting $\rho(\pi) = \frac{1}{\varepsilon}\mu(\pi)^2 + \ell^2$ be the \emph{policy-dependent bound}, the Cantelli constraint can thus be written as (\cref{app:cpo-cantelli})
\begin{equation}
    \label{eq:cpo-cheb}
    B(\pi) = \beta \tilde{\mu}(\pi) + 2\ell \mu(\pi) - \rho(\pi) \le 0.
\end{equation}

\subsection{Update Step}
\label{sec:update-step}

We will update the policy with an iterative framework akin to the one used in CPO. Again, we cannot evaluate $J(\pi)$ or $B(\pi)$ over any candidate policy $\pi$, because we only sample trajectories from an initial policy $\pi_k \ne \pi$. We circumvent this by creating a policy trust region \cite{trpo} and introducing trust-region surrogates $L_{\tilde{\mu}}(\pi)$ and $ L_\rho(\pi)$ at the current policy $\pi_k$ for augmented cost return $\tilde{\mu}(\pi)$, and policy-dependent bound $\rho(\pi)$ respectively. Defining $Z = \mathbb{E}_{\substack{x \sim d^c_{\pi_k} \\ a \sim \pi}} [A_{\pi_k}^C(x, a)]$, then as per \cref{app:d-surrogate}
\begin{gather}
    L_{\tilde{\mu}}(\pi) = \tilde{\mu}(\pi_k) + \frac{1}{1 - \gamma_c} \mathbb{E}_{\substack{x \sim d^c_{\pi_k} \\ a \sim \pi}}\left[ A^{\tilde{C}}_{\pi_k}(x,a) \right], \\
    \label{eq:d-surrogate}
    L_\rho(\pi) = \rho(\pi_k) + \frac{1}{\varepsilon} \left(\frac{2 \mu(\pi_k)}{(1-\gamma_c)} Z + \frac{1}{(1 - \gamma_c)^2} Z^2\right).
\end{gather}
Taking the predefined trust-region surrogates for the reward return in
\cref{eq:return-surrogate} and the cost return in \cref{eq:lmu}, a
trust-region surrogate for $B(\pi)$ at $\pi_k$ is given
\begin{equation}
    L_B(\pi) = \beta L_{\tilde{\mu}}(\pi) + 2\ell L_\mu(\pi) - L_\rho(\pi).
\end{equation}
This yields the final practical OP, including the trust-region surrogate of the mean constraint:
\begin{equation}
\label{eq:varcpo-constraint}
\boxed{
\begin{aligned}
    \pi_{k+1} = \arg \max_{\pi} & \text{ } L(\pi) \\
    \text{s.t.} \text{ } & L_B(\pi) \le 0, \quad L_\mu(\pi) \le \ell\\
    & \bar{D}_{KL}(\pi, \pi_k) \le \delta.
\end{aligned}
}
\end{equation}

\subsection{Worst-Case Violation Bounds}
Extending the theoretical analysis of CPO, we derive a bound on the worst-case
constraint violation introduced by the trust-region surrogates.

\begin{theorem}[Worst-Case Cantelli Violation]
\label{th:worst-case-violation} Let solution policy $\pi_{k+1}$ satisfy
OP \ref{eq:varcpo-constraint}. Then $\pi_{k+1}$ also satisfies:
\begin{equation}
    B(\pi_{k+1}) \le K \Big( \beta \alpha_{\pi_{k+1}}^{\tilde{C}} + 2 |\ell| \alpha_{\pi_{k+1}}^{C} + \frac{2\alpha_{\pi_{k+1}}^{C}}{\varepsilon} M \Big)
\end{equation}
and
\begin{equation}
\label{eq:mean-overshoot}
    \mu(\pi_{k+1}) \le \ell + K \alpha_{\pi_{k+1}}^{C},
\end{equation}
where $\alpha^{\tilde{C}}_\pi = \max_x|\mathbb{E}_{a\sim \pi}[A^{\tilde{C}}_{\pi_k}(x, a)]|$ and \\ \noindent$\alpha_\pi^C = \max_x|\mathbb{E}_{a\sim \pi}[A^C_{\pi_k}(x, a)]|$ represent the maximum expected augmented and standard cost advantages respectively over augmented states $x$ (\ref{eq:aug-state}), $K = \frac{\sqrt{2\delta}\gamma_c}{(1-\gamma_c)^2}$, and $M = |\mu(\pi_k)| + \frac{\alpha_{\pi_{k+1}}^{C}}{1-\gamma_c}$. See Appendix \ref{app:theorem} for a proof.
\end{theorem}

\cref{th:worst-case-violation} gives constraint-violation magnitude bounds for the Cantelli and mean constraints in OP \ref{eq:cantelli-objective} during training under almost surely bounded costs. However, as in CPO, the $(1 - \gamma_c)^{-2}$ dependence means the bound loosens as $\gamma_c \to 1$, unless the trust region $\delta$ shrinks correspondingly.

Since the trust-region constrained $L(\pi)$ objective in OP \ref{eq:varcpo-constraint} is
identical to CPO, it also inherits the worst case reward performance bound under almost surely bounded rewards
\cite{cpo}. Moreover, as in the CPO paper, our bounds omit accounting for any
error due to the practical necessity of estimating the advantage functions or
expected cost returns from policy roll-outs.

\subsection{Recovery Mechanism}
\label{sec:recovery}

OP \ref{eq:varcpo-constraint} may admit no solution given the trust region size. Extending CPO's treatment of infeasible updates, we replace an infeasible update with a recovery step. Writing $\Pi_\delta = \{\pi : \bar{D}_{KL}(\pi, \pi_k) \le \delta\}$ and $\Pi_\delta^\ell = \Pi_\delta \cap \{\pi : L_\mu(\pi) \le \ell\}$, define the reachable minima
\begin{equation}
\label{eq:reachability}
    m_\mu = \min_{\pi \in \Pi_\delta} L_\mu(\pi), \qquad
    m_B = \min_{\pi \in \Pi_\delta^\ell} L_B(\pi),
\end{equation}
both of which admit closed forms (\cref{app:algorithm}). The update selection has three tiers: the nominal update (a), Cantelli restoration (b), and mean restoration (c),
\begin{equation}
\label{eq:recovery}
    \pi_{k+1} =
    \begin{cases}
        (\text{a})\text{ solution of OP \ref{eq:varcpo-constraint}} & m_\mu \le \ell,\, m_B \le 0, \\
        (\text{b})\text{ }\arg \min_{\pi \in \Pi_\delta^\ell} L_B(\pi) & m_\mu \le \ell,\, m_B > 0, \\
        (\text{c})\text{ }\arg \min_{\pi \in \Pi_\delta} L_\mu(\pi) & m_\mu > \ell,
    \end{cases}
\end{equation}
equivalently evaluated as: if $m_\mu > \ell$ apply (c); else if $m_B > 0$ apply (b); else apply (a). Prioritizing safety, Tier c precedes the others because the mean constraint is a precondition for the Cantelli constraint.

No solution policy of any update tier guarantees satisfaction of the VaR constraint. Under almost surely bounded costs, Tier a alone admits the worst-case violation magnitude bounds of \cref{th:worst-case-violation}, while policies of Tiers a--b do not provably violate the VaR constraint if their cost standard deviation is large enough (\cref{cor:recovery}):
\begin{corollary}[Conditional VaR Violation]
\label[corollary]{cor:recovery}
Let $V = \{\pi: \mu(\pi) > \ell + \sigma(\pi)/\sqrt{\beta}\}$ be the region where the VaR constraint is provably violated and $\pi_{k+1}$ be produced by the selection rule \ref{eq:recovery}, with $K$ and $\alpha^{C}_{\pi_{k+1}}$ as in \cref{th:worst-case-violation}. For Tiers a--b, $\mu(\pi_{k+1}) \le \ell + K\alpha^{C}_{\pi_{k+1}}$, and consequently $\pi_{k+1} \notin V$ whenever $\sigma(\pi_{k+1}) \ge \sqrt{\beta}\, K \alpha^{C}_{\pi_{k+1}}$. For Tier c, $\mu(\pi_{k+1}) \le \mu(\pi_k) + K\alpha^{C}_{\pi_{k+1}}$. See \cref{app:corollary} for a proof.
\end{corollary}

\subsection{Practical Algorithm}

For a policy parameterized by $\theta$, given the Cantelli-VaR constrained
OP \ref{eq:varcpo-constraint} might be a non-linear problem. Therefore, in analogy to CPO, it is made computationally
tractable using Taylor expansions. The reward objective and cost constraint
are approximated to first order, while the KL divergence constraint is
approximated to second order. To this end, let
\begin{gather}
    \label{eq:taylor-c}
    c_B = L_B(\theta_k) = \beta \tilde{\mu}(\theta_k) + 2\ell \mu(\theta_k) - \rho(\theta_k)\\
    \begin{split}
    \label{eq:taylor-b}
    b_B &= \nabla L_B(\theta_k)\\
    &= \beta \nabla L_{\tilde{\mu}}(\theta_k) + \left(2\ell - \frac{2\mu(\theta_k)}{\varepsilon}\right)\nabla L_\mu(\theta_k),
    \end{split}\\
    \label{eq:taylor-cmu}
    g = \nabla L(\theta_k), \quad c_\mu = \mu(\theta_k) - \ell, \quad b_\mu = \nabla L_\mu(\theta_k).
\end{gather}
Then the problem is mapped to the following quadratically constrained linear program
\begin{equation}
\begin{aligned}
\label{eq:practical-update}
    \theta_{k+1} = \arg \max_{\theta} \text{ } & g^\top (\theta - \theta_k) \\
    \text{s.t.} \text{ } & c_B + b_B^\top (\theta - \theta_k) \le 0 \\
    & c_\mu + b_\mu^\top (\theta - \theta_k) \le 0 \\
    & \frac{1}{2} (\theta - \theta_k)^\top H (\theta - \theta_k) \le \delta
\end{aligned}
\end{equation}
where $H$ is the Hessian of the KL Divergence. The solution to OP \ref{eq:practical-update} is closed-form (\cref{app:algorithm}). However, to ensure satisfaction of OP \ref{eq:varcpo-constraint}, we implement a backtracking line search as in CPO.

\section{Relative Variance Scaling}
\label{sec:relative-variance}

We have described three formulations for the original VaR
constraint in OP \ref{eq:var-bound}: the indicator constraint in
Equation \ref{eq:indicator}, and two safe approximations, the CVaR constraint in OP
\ref{eq:cvar-opt} and the Cantelli-VaR constraint in OP \ref{eq:cantelli-objective},
which Canary is centered on. A natural question arises: \emph{what benefit does
the Cantelli constraint provide beyond its conservatism?}

In each case, a constrained OP must be solved, requiring the estimation of the \emph{constraint value} and its \emph{policy gradient}. For example, Canary requires the estimation of constraint values $c_B$ and $c_\mu$ and their policy gradients $b_B$ and $b_\mu$. For the time being, we
assume unbiased Monte Carlo estimators are used without a critic.

Given a
metric $m$ and its estimator $\hat{m}$, analyzing the total variance $\text{Var} [\hat
m] = \mathbb{E}[\vert{}\vert{}\hat{m} - \mathbb{E}[\hat{m}]\vert{}\vert{}_2^2]$ in isolation is misleading because it might shrink as the signal $m$ itself
shrinks \cite{rubinstein2016simulation}.
Therefore, the \emph{relative variance} (inverse signal-to-noise ratio)
\cite{rubinstein2016simulation} defined by
\begin{equation}
\label{eq:relative-variance}
    \kappa^2[m] = \frac{\text{Var}[\hat{m}]}{||m||_2^2}
\end{equation}
is often used instead. In deep RL, the explosion of this relative variance often manifests as a collapse in
the cosine similarity between the empirical and true values, a phenomenon
empirically identified as a driver of optimization failure in policy gradient
methods \cite{ilyas2020closerlookdeeppolicy}. We call a method \emph{stable} if, for a fixed policy, the relative variances of its constraint value and policy gradient remain bounded as the violation threshold becomes strict, $\varepsilon \to 0$.

Throughout this section and \cref{app:scaling}, we assume costs are non-negative, $c_t \ge 0$ almost surely, so $C(\tau) \ge 0$. This permits an atom at zero, $\mathbb P(C(\tau) = 0) = 1 - f$, and defines cost regimes by $f = \mathbb P(C(\tau) > 0)$: the cost signal is \emph{dense} if $f > \varepsilon$ and \emph{sparse} otherwise. The results extend to signed costs. Replacing $\mathbb P(C(\tau) > 0)$ by $\mathbb P(C(\tau) \ne 0)$ throughout, the sparse regime is empty, absent an atom at zero. The scaling results further rely on four assumptions, stated here for reference and discussed in \cref{app:scaling}.

\begin{assumption}[Positive thresholds and regular quantile]
\label[assumption]{ass:support}
The cost limit, $\ell$, and the $(1-\varepsilon)$-quantile of the cost
return, $\eta^*$ (\ref{eq:strict-var}), are positive: $\ell > 0$ and $\eta^* > 0$. Consequently, $p \le f$ and
$\varepsilon < f$, where $p = \mathbb{P}(C(\tau) > \ell)$ is the
violation probability. Moreover, the cost return has no atom at $\eta^*$, $\mathbb{P}(C(\tau) = \eta^*) = 0$, so $\mathbb{P}(C(\tau) > \eta^*) = \varepsilon$.
\end{assumption}

\begin{assumption}[Bounded conditional relative second moment]
\label[assumption]{ass:snr}
Let $X \in \{h(\tau),s(\tau) \,h(\tau) \}$,
where $s(\tau) = \sum_t \nabla_\theta \log
\pi_\theta(a_t \mid s_t)$ is the score of the trajectory, and $h(\tau) = \mathbb{I}(C(\tau) > \ell)$ for the indicator, $h(\tau) = (C(\tau) -
\eta^*)^+$ for CVaR, and $h(\tau) \in \{C(\tau), C(\tau)^2\}$ for the Cantelli moments. For
$\mathcal{E} = \{h(\tau) \ne 0\}$:
$ \mathbb{E}[\|X\|^2 \mid \mathcal{E}] \, / \,\|\mathbb{E}[X \mid \mathcal{E}]\|^2 \le K$ for a constant $K$ independent of $p$, $\varepsilon$, and $f$.
\end{assumption}

\begin{assumption}[No catastrophic cancellation]
\label[assumption]{ass:no-cancel}
Given $c_1 = 2(\ell - \mu(\theta)/\varepsilon)$ and $c_2 = \beta$. There exists a $\nu > 0$, independent of $\theta$ and $(p, \varepsilon, f)$, such that $\left\| c_1 x_1 + c_2 x_2 \right\|^2 \ge \nu ( c_1^2 \left\| x_1 \right\|^2 + c_2^2 \left\| x_2 \right\|^2 )$ for both the moment pair $(x_1, x_2) = (\mu(\theta), \tilde{\mu}(\theta))$ and the moment-gradient pair $(x_1, x_2) = (\nabla_\theta \mu(\theta), \nabla_\theta \tilde{\mu}(\theta))$.
\end{assumption}

\begin{assumption}[Proportional slack]
\label[assumption]{ass:slack}
Each constraint value decomposes as $J_C(\theta) = \mathbb{E}[x(\tau)] + b$, a \emph{sampled part} $\mathbb{E}[x(\tau)]$, where the sampled statistic $x(\tau)$ vanishes off the active event $\mathcal{E}$, plus a batch-independent offset $b$. The \emph{signal share} $r = |\mathbb{E}[x(\tau)]| / |J_C(\theta)|$ satisfies $r_{\min} \le r \le r_{\max}$ for constants $0 < r_{\min} \le r_{\max} < \infty$ independent of $(p, \varepsilon, f)$.
For the Cantelli and mean constraints, for $X \in \{x(\tau), s(\tau)\,x(\tau)\}$, $\mathbb{E}[\|X\|^2 \mid \mathcal{E}] \, / \, \|\mathbb{E}[X \mid \mathcal{E}]\|^2 \ge 1 + \omega$ for a constant $\omega > 0$ independent of $(p, \varepsilon, f)$.
\end{assumption}

\begin{proposition}[Relative variance of constraints]
  \label[proposition]{prop:ess}
  Under Assumptions \ref{ass:support} to \ref{ass:slack}, for unbiased Monte Carlo estimators, the relative variance of the
  constraint value and gradient over a fixed batch of $N$ trajectories satisfies the following bounds (\cref{app:scaling}):
  \begin{enumerate}[label=(\roman*),nosep]
    \item \emph{Indicator:} $\Theta\!\left(\frac{1}{pN}\right)$ \cite{rubinstein2016simulation};
    \item \emph{CVaR:} $\Theta\!\left(\frac{1}{\varepsilon N}\right)$
    \cite{tamar2014optimizingcvarsampling};
    \item \emph{Cantelli:} $\Theta\!\left(\frac{1}{fN}\right)$ (for signed costs, $f = \mathbb P(C(\tau) \ne 0)$, which is $1$ absent an atom at zero).
  \end{enumerate}
\end{proposition}
The unbiased Cantelli estimators analyzed are not computable, so the relative variances of the biased \emph{plug-in estimators} of $c_B$ and $b_B$ that Canary uses, whose bias shrinks as $1/N$, are also analyzed in \cref{app:scaling} under one additional assumption (\ref{eq:cross-moment}). The constraint value keeps the $\Theta(1/(fN))$ rate to leading order and the gradient an $\mathcal{O}(1/(fN))$ upper bound. Therefore, only the Cantelli estimator is stable given dense cost regimes. The next section shows how this translates to improved safety for low $\varepsilon$ regimes.

\section{Results}
\label{sec:results}
Although the Safety Gymnasium package provides an excellent set of benchmark
environments for CMDPs \cite{safety-gymnasium}, we leverage the benefits of
running Just-in-Time (JIT) compiled JAX code end-to-end on the GPU for
accelerated experiments \cite{jax}. This required translating well-known
benchmarks:
\begin{itemize}[nosep]
    \item \textbf{EcoAnt}: A reimplementation of the Ant environment \cite{gae}
    in MJX. In this scenario, the agent must maximize forward velocity, while
    now managing a limited battery budget ($c_t = 0.5 ||a_t||^2_2$).
    \item \textbf{PointGoal}: A reimplementation of SafetyPointGoal1 from Safety
    Gymnasium \cite{safety-gymnasium} built on top of MJX. In this scenario, the
    agent must navigate toward a randomized target location while avoiding
    static hazardous areas distributed across the arena.
    \item \textbf{IcyLake}: A reimplementation of the gymnasium FrozenLake
    environment \cite{gymnasium} in JAX. The agent must traverse a $4 \times 4$
    frozen lake grid to reach the goal state while avoiding icy tiles which
    accumulate cost $+1$ per visit.
\end{itemize}

We compare \textbf{Canary} against four baselines:
\begin{itemize}[nosep]
    \item \textbf{Proximal Policy Optimization (PPO):} A popular unconstrained
    RL baseline \cite{ppo}.
    \item \textbf{Constrained Policy Optimization (CPO):} A trust-region method
    for CMDPs. Here, the indicator constraint is used as discussed in Section
    \ref{sec:indicator} to enforce a VaR constraint (OP \ref{eq:var-bound}) \cite{cpo}.
    \item \textbf{CVaR Proximal Policy Optimization (CPPO):} A Lagrange
    augmented PPO method that enforces a CVaR constraint (OP \ref{eq:cvar-opt}) \cite{cvar-ppo}.
    \item \textbf{CVaR Constrained Policy Optimization (CVaR-CPO):} A
    trust-region method that enforces a CVaR constraint using a distributional cost critic \cite{cvar-cpo}.
\end{itemize}

In practice, trajectories terminate or are truncated after a finite time $T_\tau$, so $\mathcal{N} = \{0, \dots, T_\tau - 1\}$. Further details on experimental setup can be found in Appendix \ref{app:hyperparams}.

\subsection{Empirical Relative Variance Scaling}

\cref{fig:relative-variance} is consistent with the scaling behavior of \cref{prop:ess}, even with trust-region surrogates using critic-dependent GAE advantage estimates for policy gradients. The relative variances of the
indicator and CVaR constraints grow as $\varepsilon\to0$ while the Cantelli constraint
remains flat.
\begin{figure}[H]
    \centering
    \includegraphics[width=0.78\linewidth]{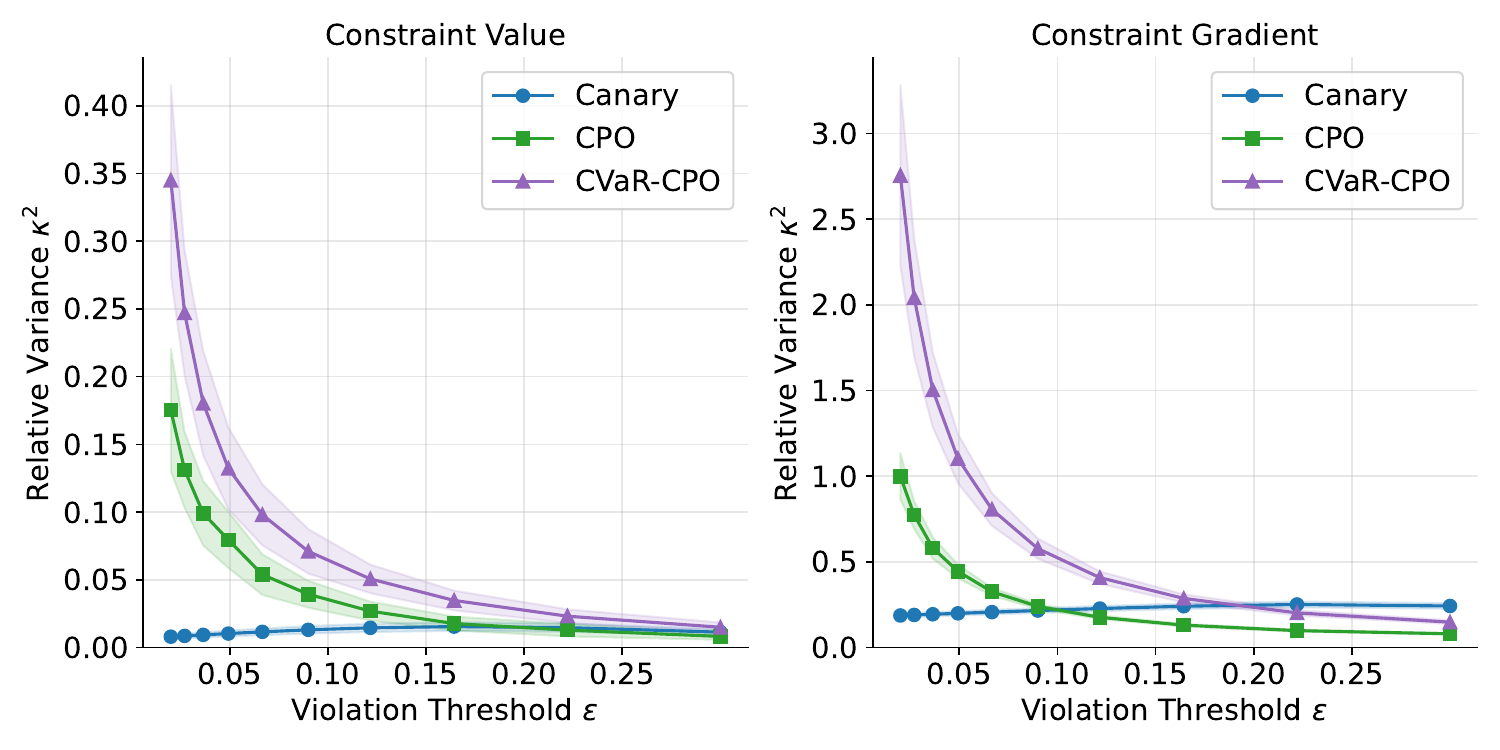}
    \caption{\textbf{Relative Variance Scaling:} IcyLake relative variance
    scaling for the constraint value (left) and constraint gradient (right) for
    Canary (Cantelli constraint) (blue), CPO (indicator) (green), and CVaR-CPO (purple). Each method is using a critic-dependent GAE advantage estimate for the policy gradient. Lines are means and shaded areas are
    one standard deviation over 30 seeds.}
    \label{fig:relative-variance}
\end{figure}

\subsection{Training Performance}
We select violation threshold $\varepsilon = 0.02$ and cost limit $\ell$ for feasible settings in each environment (see \cref{tab:safety}) with undiscounted costs $\gamma_c = 1$, which the finite $\mathcal{N}$ permits.

We assess a method's VaR constraint satisfaction throughout training by its \emph{reliability} ($\%$ safe, \cref{tab:safety}). \cref{tab:safety} shows that Canary is the most reliable method across all environments, and two measured properties contribute to this result. First, a method's converged violation probability ($\hat P$, \cref{tab:safety}) measures how far below the violation threshold the method converges on average; and second, its violation probability jitter ($\hat\sigma$, \cref{tab:safety}) measures how much the violation probability varies at convergence. Reliability then follows from the margin per unit jitter, $(\varepsilon - \hat P)/\hat\sigma$, which for Canary is the largest of all methods in every environment. The theory supports these results. The Cantelli-VaR constraints used by Canary are the most conservative surrogate (\cref{sec:cantelli-var}), which contributes to its low converged violation probability. The Cantelli-VaR constraints are also stable in low violation threshold scenarios (\cref{prop:ess}), which contributes to its low violation probability jitter.

\begin{table}[t]
  \centering
  \small
  \setlength{\tabcolsep}{2.4pt}
      \begin{tabular}{llcccc}
    \toprule
    & Method & $\hat{P}\downarrow$ & $\hat{\sigma}\downarrow$ & \% safe$\uparrow$ & $\hat J\uparrow$ \\
    \midrule
    \multirow{5}{*}{\rotatebox{90}{\shortstack{EcoAnt\\$\ell{=}150$}}}
      & Canary   & $\mathbf{.002}\,{\pm}.000$ & $\mathbf{1.4}\,{\pm}0.2$ & $\mathbf{100}$ & $81{\pm}3$ \\
      & PPO      & $.694\,{\pm}.003$ & $36.6\,{\pm}2.1$ & $0$ & $1321{\pm}7$ \\
      & CPO      & $.030\,{\pm}.000$ & $4.2\,{\pm}0.4$ & $0$  & $203{\pm}6$ \\
      & CPPO     & $.048\,{\pm}.003$ & $20.9\,{\pm}2.2$ & $3$  & $48{\pm}5$ \\
      & CVaR-CPO & $.038\,{\pm}.002$ & $16.8\,{\pm}3.6$ & $0$  & $123{\pm}5$ \\
    \midrule
    \multirow{5}{*}{\rotatebox{90}{\shortstack{PointGoal\\$\ell{=}100$}}}
      & Canary   & $\mathbf{.011}\,{\pm}.001$ & $\mathbf{6.1}\,{\pm}1.4$ & $\mathbf{87}$ & $4.0{\pm}0.6$ \\
      & PPO      & $.512\,{\pm}.006$ & $23.1\,{\pm}4.2$ & $0$ & $69.0{\pm}0.0$ \\
      & CPO      & $.060\,{\pm}.002$ & $31.0\,{\pm}2.6$ & $0$  & $12.1{\pm}1.2$ \\
      & CPPO     & $.045\,{\pm}.010$ & $22.2\,{\pm}4.0$ & $47$  & $7.3{\pm}1.6$ \\
      & CVaR-CPO & $.015\,{\pm}.001$ & $11.8\,{\pm}1.5$ & $\mathbf{87}$  & $6.4{\pm}2.0$ \\
    \midrule
    \multirow{5}{*}{\rotatebox{90}{\shortstack{IcyLake\\$\ell{=}5.5$}}}
      & Canary   & $\mathbf{.016}\,{\pm}.007$ & $\mathbf{7.1}\,{\pm}2.7$ & $\mathbf{87}$ & $.67{\pm}.06$ \\
      & PPO      & $.308\,{\pm}.002$ & $19.7\,{\pm}1.2$ & $0$ & $1.0{\pm}.00$ \\
      & CPO      & $.041\,{\pm}.003$ & $9.9\,{\pm}0.9$ & $7$  & $.96{\pm}.00$ \\
      & CPPO     & $.354\,{\pm}.002$ & $27.1\,{\pm}1.5$ & $0$ & $1.0{\pm}.00$ \\
      & CVaR-CPO & $.020\,{\pm}.002$ & $7.5\,{\pm}1.3$ & $53$  & $.93{\pm}.01$ \\
    \bottomrule
  \end{tabular}
  \caption{Safety adherence across environments and methods ($\varepsilon{=}0.02$, 30
  seeds, mean$\pm$standard error over seeds). \emph{Converged violation probability} $\hat{P}$: per-seed average violation probability over the second half of training;
  \emph{violation-probability jitter} $\hat{\sigma}$: per-seed standard deviation of the violation probability over the second half of training ($\times10^{-3}$; \cref{app:jitter}); \emph{reliability} \% safe: percent of seeds $s$ whose converged violation probability
  satisfies $\hat{P}^{(s)} \le \varepsilon$; \emph{converged reward return} $\hat J$: per-seed average undiscounted reward return over the second half of training. Note, Canary settles at the lowest $\hat P$ and $\hat\sigma$ in every environment.}
  \label{tab:safety}
\end{table}
\begin{figure}[t]
    \centering
    \includegraphics[width=0.85\linewidth]{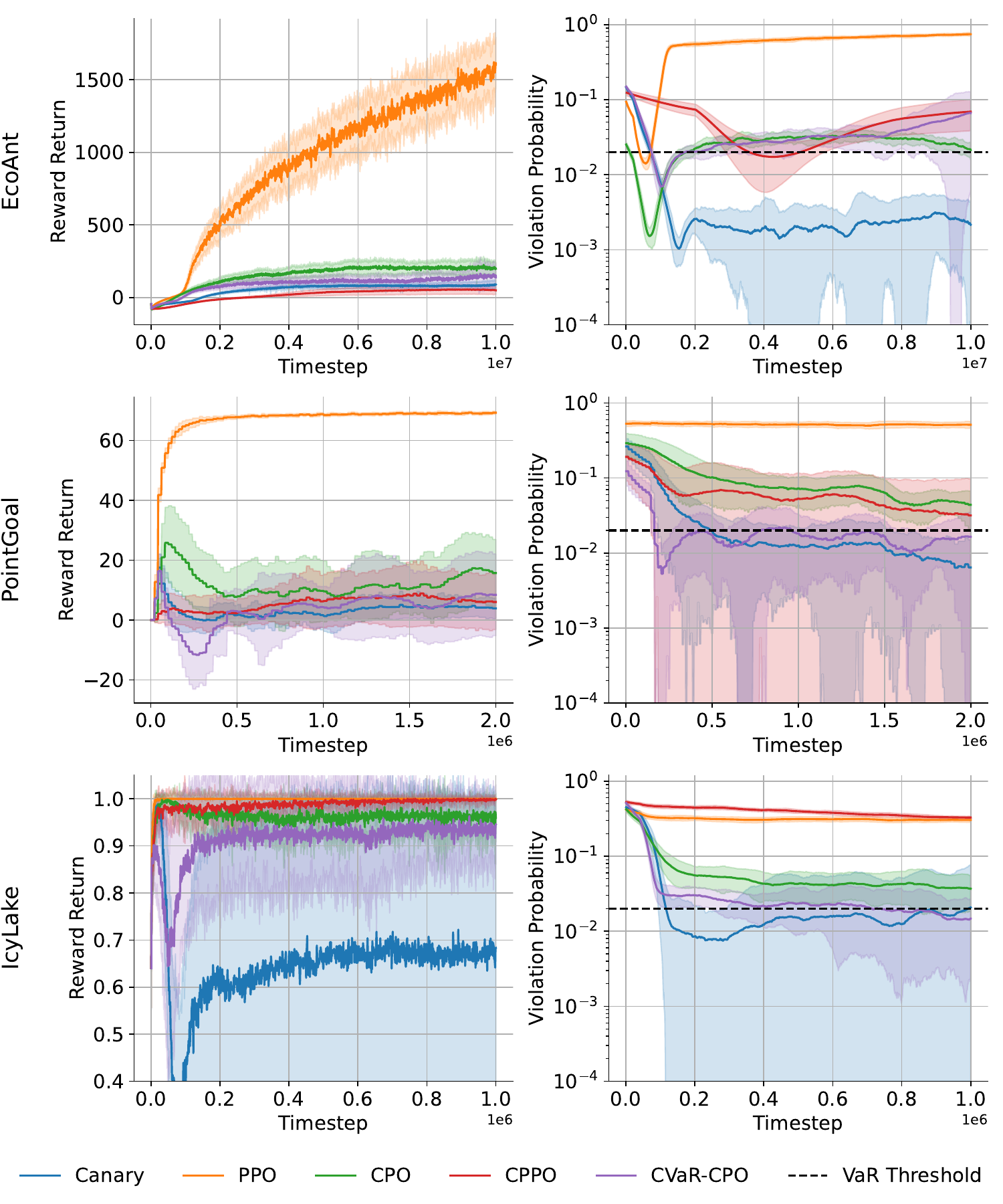}
    \caption{Reward return (left) and violation probability (right; log
    scale) during training on EcoAnt ($\ell{=}150$, top), PointGoal
    ($\ell{=}100$, middle) and IcyLake ($\ell{=}5.5$, bottom) under the VaR
    constraint $\mathbb{P}(C(\tau) > \ell) \le 0.02$. Lines are means over
    30 seeds; shaded areas one standard deviation.}
    \label{fig:benchmark}
\end{figure}

\subsection{Canary Violation Threshold ($\varepsilon$) Sensitivity}
\begin{figure}[t]
    \centering
    \begin{subfigure}[b]{0.46\linewidth}
        \centering
        \includegraphics[width=\linewidth]{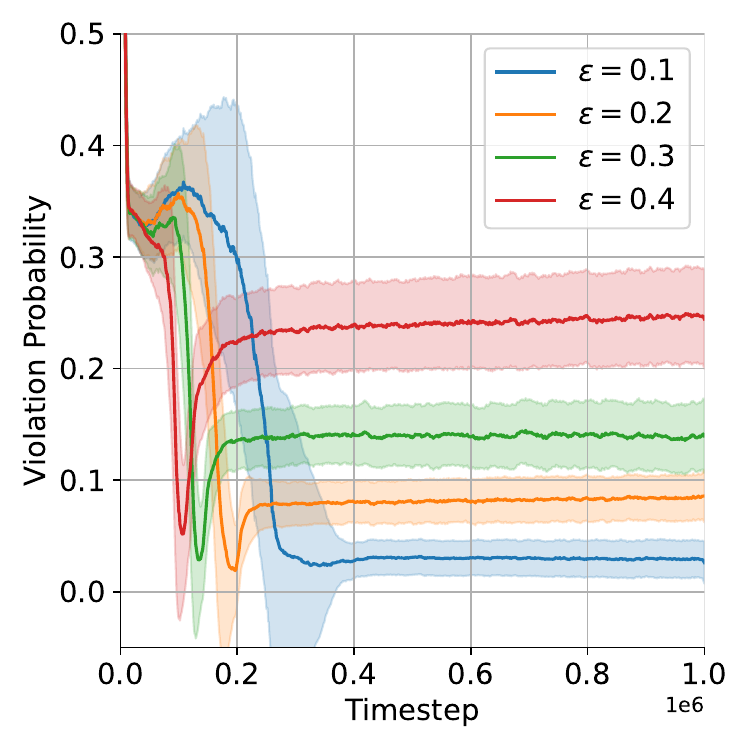}
        \caption{Training Dynamics}
        \label{fig:epsilon_ablation}
    \end{subfigure}\hfill
    \begin{subfigure}[b]{0.46\linewidth}
        \centering
        \includegraphics[width=\linewidth]{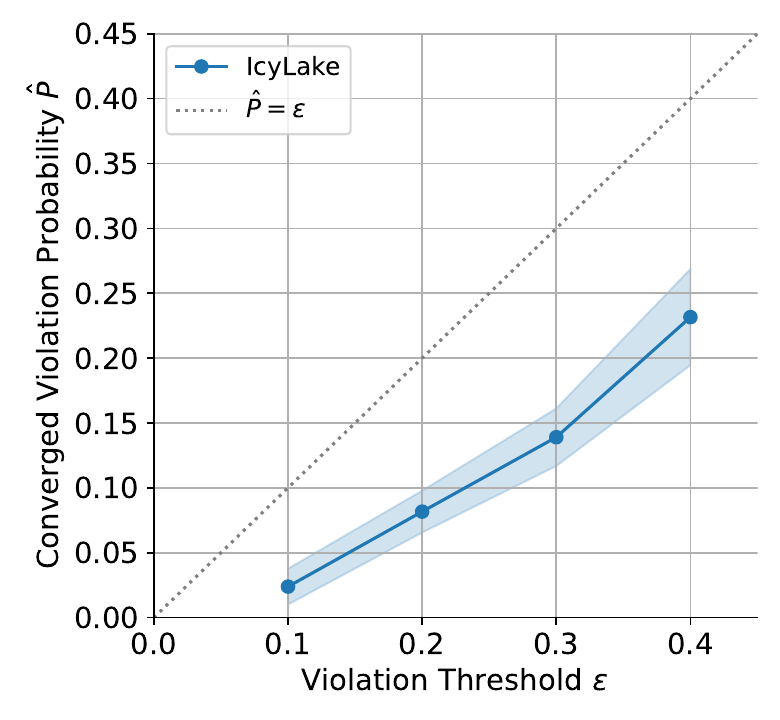}
        \caption{Converged Calibration}
        \label{fig:eps-calibration}
    \end{subfigure}
    \caption{\textbf{Violation Threshold Sensitivity:} (a) Probability of a cost
    return exceeding $\ell= 15$ in IcyLake with varying violation thresholds $\varepsilon$. Lines are mean and shaded
    areas represent one standard deviation over 500 seeds. (b) The same sweep with violation threshold vs the converged
    violation probability $\hat P$.}
    \label{fig:eps_sweep}
\end{figure}
For Canary, we conduct a sensitivity analysis on the violation threshold $\varepsilon$ in
Figure \ref{fig:epsilon_ablation}. Excluding the plateau around violation probability $0.35$, the rate of convergence is independent of $\varepsilon$. The collapse from violation probability $0.3$ to $0.1$ takes a median of roughly $20$k steps at every $\varepsilon$. \cref{fig:eps-calibration} further highlights the conservative nature of
Canary established in \cref{sec:cantelli-var}, with empirical
expected violation rates settling $2$--$4\times$ below the thresholds set.

\subsection{Anatomy of the Conservatism}
\label{sec:anatomy}
The converged violation probabilities in \cref{tab:safety} settle well below the
violation threshold $\varepsilon$, at some cost in reward. This conservatism
could stem from \emph{optimizer slack} (converging to the interior of the
Cantelli-feasible set, forgoing reward the Cantelli constraint would permit) or
\emph{inequality slack} (the conservatism of the distribution-free Cantelli constraint itself). \cref{fig:anatomy} separates the two regimes: in
EcoAnt and PointGoal every seed converges with $c_B > 0$,
so the margin there is inequality slack, the regime whose training-time violation
\cref{th:worst-case-violation} bounds. On IcyLake most seeds (22/30) instead
converge inside the boundary, so its conservatism stacks optimizer slack
on top of inequality slack.

\begin{remark}[Gaussian tails underestimate risk]
\label[remark]{rem:gaussian}
It is common to model the cost return as Gaussian \cite{wcsac}. At the converged cost return moments
(mean $\mu$ and standard deviation $\sigma$), median $z = (\ell - \mu) /
\sigma$ is $6.1$ for EcoAnt and $5.0$ for PointGoal, so a Gaussian tail
predicts violation probabilities on the order of $10^{-9}$ and $10^{-7}$, which are
$2{\times}10^{6}$ times (EcoAnt) and $3{\times}10^{4}$ times (PointGoal) below
the realized rates, an understatement of the exceedance probability by up to
six orders of magnitude.
\end{remark}
\begin{figure}[t]
    \centering
    \includegraphics[width=0.46\linewidth]{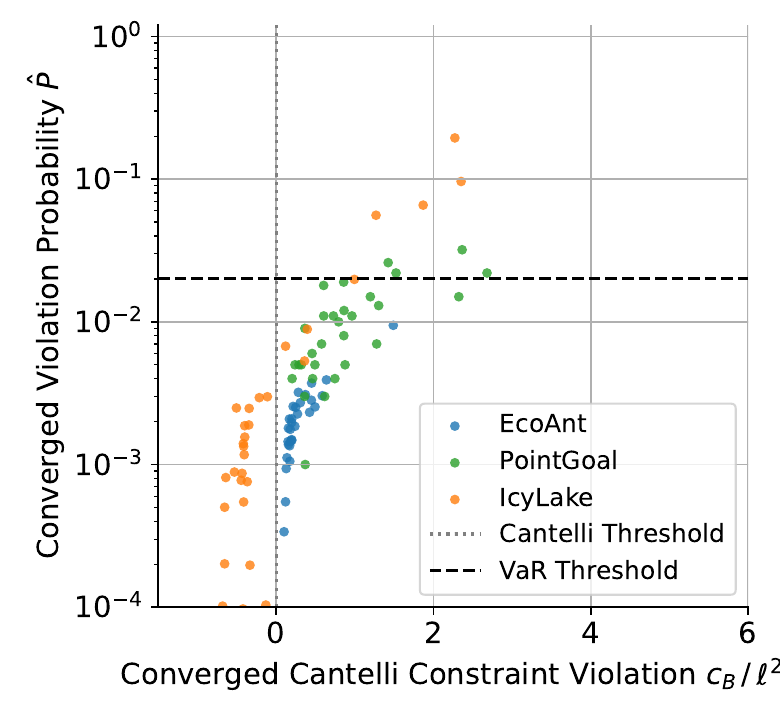}
    \caption{\textbf{Anatomy of the Conservatism}: Converged Cantelli
    constraint violation $c_B/\ell^2$ (rescaled by $\ell^2$ for plotting purposes, per-seed mean over the second half of
    training)
    against the converged violation probability $\hat P$ for EcoAnt (blue), PointGoal (green)
    and IcyLake (orange) at $\varepsilon = 0.02$ (30 seeds each, log scale).
    Seeds right of the Cantelli threshold ($c_B > 0$) carry inequality slack
    only; seeds strictly inside the boundary add optimizer slack
    (\cref{sec:anatomy}). 82 of 90 seeds have converged violation
    probabilities below the VaR threshold.}
    \label{fig:anatomy}
\end{figure}

\section{Conclusion}
\label{sec:conclusion}

We introduced Canary, an on-policy RL method that satisfies
VaR constraints. First, we formulated a smooth and tractable safe approximation of
VaR constraints from the first two moments of the cost return via Cantelli's
inequality, extending CPO's worst-case violation bounds to the Canary update
step. Second, we showed that Canary remains stable at tight violation thresholds in dense cost regimes. Finally, we found empirically that Canary attains the lowest converged
violation probability and the least violation-probability jitter of the
methods tested, contributing to its reliability across seeds.

\paragraph{Limitations.} \looseness=-1 The Cantelli-VaR constraint can conservatively reject higher reward policies that satisfy the original VaR constraint (\cref{sec:anatomy}), and settings infeasible for it may remain VaR-feasible. Canary's stability requires a dense cost signal (\cref{sec:relative-variance}). The worst-case violation bound of \cref{th:worst-case-violation} omits advantage and moment estimation errors, and $\gamma_c \rightarrow 1$ loosens it unboundedly at fixed $\delta$.

\section*{Acknowledgments}
Rohan Tangri gratefully acknowledges support from G-Research.

\section*{Impact Statement}

This paper presents work whose goal is to advance the field of Machine Learning.
There are many potential societal consequences of our work, none which we feel
must be specifically highlighted here.


\bibliography{example_paper}

@article{geibel-wysotzki,
  author  = {Geibel, Peter and Wysotzki, Fritz},
  title   = {Risk-Sensitive Reinforcement Learning Applied to Control under Constraints},
  journal = {Journal of Artificial Intelligence Research},
  volume  = {24},
  pages   = {81--108},
  year    = {2005},
}

@misc{gae,
      title={High-Dimensional Continuous Control Using Generalized Advantage Estimation}, 
      author={John Schulman and Philipp Moritz and Sergey Levine and Michael Jordan and Pieter Abbeel},
      year={2018},
      eprint={1506.02438},
      archivePrefix={arXiv},
      primaryClass={cs.LG},
      url={https://arxiv.org/abs/1506.02438}, 
}

@book{rubinstein2016simulation,
  title={Simulation and the Monte Carlo method},
  author={Rubinstein, Reuven Y and Kroese, Dirk P},
  year={2016},
  publisher={John Wiley \& Sons}
}

@misc{ilyas2020closerlookdeeppolicy,
      title={A Closer Look at Deep Policy Gradients}, 
      author={Andrew Ilyas and Logan Engstrom and Shibani Santurkar and Dimitris Tsipras and Firdaus Janoos and Larry Rudolph and Aleksander Madry},
      year={2020},
      eprint={1811.02553},
      archivePrefix={arXiv},
      primaryClass={cs.LG},
      url={https://arxiv.org/abs/1811.02553}, 
}

@misc{tamar2014optimizingcvarsampling,
      title={Optimizing the CVaR via Sampling}, 
      author={Aviv Tamar and Yonatan Glassner and Shie Mannor},
      year={2014},
      eprint={1404.3862},
      archivePrefix={arXiv},
      primaryClass={stat.ML},
      url={https://arxiv.org/abs/1404.3862}, 
}

@misc{reinforce,
      title={Sample Efficient Reinforcement Learning with REINFORCE}, 
      author={Junzi Zhang and Jongho Kim and Brendan O'Donoghue and Stephen Boyd},
      year={2020},
      eprint={2010.11364},
      archivePrefix={arXiv},
      primaryClass={cs.LG},
      url={https://arxiv.org/abs/2010.11364}, 
}

@misc{rldriving-survey,
      title={Deep Reinforcement Learning for Autonomous Driving: A Survey}, 
      author={B Ravi Kiran and Ibrahim Sobh and Victor Talpaert and Patrick Mannion and Ahmad A. Al Sallab and Senthil Yogamani and Patrick Pérez},
      year={2021},
      eprint={2002.00444},
      archivePrefix={arXiv},
      primaryClass={cs.LG},
      url={https://arxiv.org/abs/2002.00444}, 
}

@misc{rlfin-survey,
      title={Recent Advances in Reinforcement Learning in Finance}, 
      author={Ben Hambly and Renyuan Xu and Huining Yang},
      year={2023},
      eprint={2112.04553},
      archivePrefix={arXiv},
      primaryClass={q-fin.MF},
      url={https://arxiv.org/abs/2112.04553}, 
}

@article{var-cheb,
   title={Multivariate Chebyshev Inequality With Estimated Mean and Variance},
   volume={71},
   ISSN={1537-2731},
   url={http://dx.doi.org/10.1080/00031305.2016.1186559},
   DOI={10.1080/00031305.2016.1186559},
   number={2},
   journal={The American Statistician},
   publisher={Informa UK Limited},
   author={Stellato, Bartolomeo and Van Parys, Bart P. G. and Goulart, Paul J.},
   year={2017},
   month=apr, pages={123–127} }

@misc{trpo,
      title={Trust Region Policy Optimization}, 
      author={John Schulman and Sergey Levine and Philipp Moritz and Michael I. Jordan and Pieter Abbeel},
      year={2017},
      eprint={1502.05477},
      archivePrefix={arXiv},
      primaryClass={cs.LG},
      url={https://arxiv.org/abs/1502.05477}, 
}

@book{cmdp,
    author = {Erik Altman},
    title = {Constrained Markov Decision Processes},
    publisher = {Chapman \& Hall/CRC},
    year = {1999}
}

@misc{lagrange-constrained,
      title={Reward Constrained Policy Optimization}, 
      author={Chen Tessler and Daniel J. Mankowitz and Shie Mannor},
      year={2018},
      eprint={1805.11074},
      archivePrefix={arXiv},
      primaryClass={cs.LG},
      url={https://arxiv.org/abs/1805.11074}, 
}

@misc{cpo,
      title={Constrained Policy Optimization}, 
      author={Joshua Achiam and David Held and Aviv Tamar and Pieter Abbeel},
      year={2017},
      eprint={1705.10528},
      archivePrefix={arXiv},
      primaryClass={cs.LG},
      url={https://arxiv.org/abs/1705.10528}, 
}

@ARTICLE{cvar-cpo,
  author={Zhang, Qiyuan and Leng, Shu and Ma, Xiaoteng and Liu, Qihan and Wang, Xueqian and Liang, Bin and Liu, Yu and Yang, Jun},
  journal={IEEE Transactions on Neural Networks and Learning Systems}, 
  title={CVaR-Constrained Policy Optimization for Safe Reinforcement Learning}, 
  year={2025},
  volume={36},
  number={1},
  pages={830-841},
  doi={10.1109/TNNLS.2023.3331304}}

@inproceedings{distrl,
 author = {Lim, Shiau Hong and Malik, Ilyas},
 booktitle = {Advances in Neural Information Processing Systems},
 editor = {S. Koyejo and S. Mohamed and A. Agarwal and D. Belgrave and K. Cho and A. Oh},
 pages = {30977--30989},
 publisher = {Curran Associates, Inc.},
 title = {Distributional Reinforcement Learning for Risk-Sensitive Policies},
 volume = {35},
 year = {2022}
}

@article{var-lagrange,
author = {Chow, Yinlam and Ghavamzadeh, Mohammad and Janson, Lucas and Pavone, Marco},
title = {Risk-constrained reinforcement learning with percentile risk criteria},
year = {2017},
issue_date = {January 2017},
publisher = {JMLR.org},
volume = {18},
number = {1},
issn = {1532-4435},
journal = {J. Mach. Learn. Res.},
month = jan,
pages = {6070–6120},
numpages = {51}
}

@misc{cvar-ppo,
      title={Towards Safe Reinforcement Learning via Constraining Conditional Value-at-Risk}, 
      author={Chengyang Ying and Xinning Zhou and Hang Su and Dong Yan and Ning Chen and Jun Zhu},
      year={2022},
      eprint={2206.04436},
      archivePrefix={arXiv},
      primaryClass={cs.LG},
      url={https://arxiv.org/abs/2206.04436}, 
}

@misc{ppo,
      title={Proximal Policy Optimization Algorithms}, 
      author={John Schulman and Filip Wolski and Prafulla Dhariwal and Alec Radford and Oleg Klimov},
      year={2017},
      eprint={1707.06347},
      archivePrefix={arXiv},
      primaryClass={cs.LG},
      url={https://arxiv.org/abs/1707.06347}, 
}

@misc{rl-hft,
      title={Deep Reinforcement Learning for Active High Frequency Trading}, 
      author={Antonio Briola and Jeremy Turiel and Riccardo Marcaccioli and Alvaro Cauderan and Tomaso Aste},
      year={2023},
      eprint={2101.07107},
      archivePrefix={arXiv},
      primaryClass={cs.LG},
      url={https://arxiv.org/abs/2101.07107}, 
}

@misc{tagawa,
    title = {Chebyshev Inequality based Approach to Chance Constrained Portfolio Optimization},
    author = {Koyoharu Tagawa},
    year={2017},
    url = {https://www.iaras.org/iaras/filedownloads/ijmcm/2017/001-0009(2017).pdf}
}

@article{cvar1,
  title={Optimization of conditional value-at risk},
  author={R. Tyrrell Rockafellar and Stanislav Uryasev},
  journal={Journal of Risk},
  year={2000},
  volume={3},
  pages={21-41},
  url={https://api.semanticscholar.org/CorpusID:854622}
}

@article{cvar2,
title = {Conditional value-at-risk for general loss distributions},
journal = {Journal of Banking \& Finance},
volume = {26},
number = {7},
pages = {1443-1471},
year = {2002},
issn = {0378-4266},
doi = {https://doi.org/10.1016/S0378-4266(02)00271-6},
url = {https://www.sciencedirect.com/science/article/pii/S0378426602002716},
author = {R.Tyrrell Rockafellar and Stanislav Uryasev}
}

@article{gymnasium,
  title={Gymnasium: A Standard Interface for Reinforcement Learning Environments},
  author={Towers, Mark and Kwiatkowski, Ariel and Terry, Jordan and Balis, John U and De Cola, Gianluca and Deleu, Tristan and Goul{\~a}o, Manuel and Kallinteris, Andreas and Krimmel, Markus and KG, Arjun and others},
  journal={arXiv preprint arXiv:2407.17032},
  year={2024}
}

@inproceedings{safety-gymnasium,
  title={Safety Gymnasium: A Unified Safe Reinforcement Learning Benchmark},
  author={Jiaming Ji and Borong Zhang and Jiayi Zhou and Xuehai Pan and Weidong Huang and Ruiyang Sun and Yiran Geng and Yifan Zhong and Josef Dai and Yaodong Yang},
  booktitle={Thirty-seventh Conference on Neural Information Processing Systems Datasets and Benchmarks Track},
  year={2023},
  url={https://openreview.net/forum?id=WZmlxIuIGR}
}

@software{jax,
  author = {James Bradbury and Roy Frostig and Peter Hawkins and Matthew James Johnson and Chris Leary and Dougal Maclaurin and George Necula and Adam Paszke and Jake Vander{P}las and Skye Wanderman-{M}ilne and Qiao Zhang},
  title = {{JAX}: composable transformations of {P}ython+{N}um{P}y programs},
  url = {http://github.com/jax-ml/jax},
  version = {0.3.13},
  year = {2018},
}

@misc{her,
      title={Hindsight Experience Replay}, 
      author={Marcin Andrychowicz and Filip Wolski and Alex Ray and Jonas Schneider and Rachel Fong and Peter Welinder and Bob McGrew and Josh Tobin and Pieter Abbeel and Wojciech Zaremba},
      year={2018},
      eprint={1707.01495},
      archivePrefix={arXiv},
      primaryClass={cs.LG},
      url={https://arxiv.org/abs/1707.01495}, 
}

@inproceedings{wcsac,
  title={{WCSAC}: Worst-Case Soft Actor Critic for Safety-Constrained Reinforcement Learning},
  author={Yang, Qisong and Sim{\~a}o, Thiago D. and Tindemans, Simon H. and Spaan, Matthijs T. J.},
  booktitle={Proceedings of the AAAI Conference on Artificial Intelligence},
  volume={35},
  number={12},
  pages={10639--10646},
  year={2021}
}

@article{elghaoui2003,
  author  = {El Ghaoui, Laurent and Oks, Maksim and Oustry, Fran{\c{c}}ois},
  title   = {Worst-Case Value-at-Risk and Robust Portfolio Optimization: A Conic Programming Approach},
  journal = {Operations Research},
  volume  = {51},
  number  = {4},
  pages   = {543--556},
  year    = {2003},
}

@article{bertsimas2005,
  author  = {Bertsimas, Dimitris and Popescu, Ioana},
  title   = {Optimal Inequalities in Probability Theory: A Convex Optimization Approach},
  journal = {SIAM Journal on Optimization},
  volume  = {15},
  number  = {3},
  pages   = {780--804},
  year    = {2005},
}

@article{zymler2013joint,
  author  = {Zymler, Steve and Kuhn, Daniel and Rustem, Ber{\c{c}}},
  title   = {Distributionally Robust Joint Chance Constraints with Second-Order Moment Information},
  journal = {Mathematical Programming},
  volume  = {137},
  number  = {1--2},
  pages   = {167--198},
  year    = {2013},
}

@article{ono-ccdp,
  author  = {Ono, Masahiro and Pavone, Marco and Kuwata, Yoshiaki and Balaram, J.},
  title   = {Chance-Constrained Dynamic Programming with Application to Risk-Aware Robotic Space Exploration},
  journal = {Autonomous Robots},
  volume  = {39},
  number  = {4},
  pages   = {555--571},
  year    = {2015},
}

@inproceedings{santana-rao,
  title={{RAO}*: An Algorithm for Chance-Constrained {POMDP}'s},
  author={Santana, Pedro and Thi{\'e}baux, Sylvie and Williams, Brian},
  booktitle={Proceedings of the AAAI Conference on Artificial Intelligence},
  volume={30},
  number={1},
  pages={3308--3314},
  year={2016}
}

@inproceedings{xu-mannor,
  author    = {Xu, Huan and Mannor, Shie},
  title     = {Probabilistic Goal {M}arkov Decision Processes},
  booktitle = {Proceedings of the Twenty-Second International Joint Conference on Artificial Intelligence},
  pages     = {2046--2052},
  year      = {2011}
}

@inproceedings{saute,
  author    = {Sootla, Aivar and Cowen-Rivers, Alexander I. and Jafferjee, Taher and Wang, Ziyan and Mguni, David H. and Wang, Jun and Ammar, Haitham},
  title     = {Saute {RL}: Almost Surely Safe Reinforcement Learning Using State Augmentation},
  booktitle = {Proceedings of the 39th International Conference on Machine Learning},
  volume    = {162},
  pages     = {20423--20443},
  year      = {2022},
  publisher = {PMLR}
}

@article{nemirovski2012safe,
  author  = {Nemirovski, Arkadi},
  title   = {On Safe Tractable Approximations of Chance Constraints},
  journal = {European Journal of Operational Research},
  volume  = {219},
  number  = {3},
  pages   = {707--718},
  year    = {2012}
}

@book{shapiro2009lectures,
  title={Lectures on Stochastic Programming: Modeling and Theory},
  author={Alexander Shapiro and Darinka Dentcheva and Andrzej Ruszczy{\'n}ski},
  year={2009},
  publisher={SIAM},
  address={Philadelphia}
}
\bibliographystyle{icml2026}

\newpage
\appendix
\onecolumn
\section{Canary Method Proofs}
\label[appendix]{app:proofs}

\subsection{Derivation of Equation \ref{eq:chebyshev_constraint}}
\label[appendix]{app:cantelli-var}

Assume first that $\mu(\pi) < \ell$, so that $\lambda = \ell - \mu(\pi) > 0$. By Cantelli's inequality:
$$
\mathbb{P}(C(\tau) - \mu(\pi) \ge \lambda) \le \frac{\sigma^2(\pi)}{\sigma^2(\pi) + \lambda^2},
$$
which is equivalent to:
$$
\mathbb{P}(C(\tau) > \ell) \le \mathbb{P}(C(\tau) \ge \ell) \le \frac{\sigma^2(\pi)}{\sigma^2(\pi) + [\ell - \mu(\pi)]^2},
$$
where the first inequality holds because $\{C(\tau) > \ell\} \subseteq \{C(\tau) \ge \ell\}$.

To guarantee that this probability is bounded by some $\varepsilon \in (0, 1]$, it is sufficient to require that the upper bound itself is less than or equal to $\varepsilon$:
$$
\frac{\sigma^2(\pi)}{\sigma^2(\pi) + [\ell - \mu(\pi)]^2} \le \varepsilon.
$$

We can now simplify this condition. Since $\lambda > 0$ and $\varepsilon > 0$, the following steps are logically equivalent:
\begin{align*}
    &\frac{\sigma^2(\pi)}{\sigma^2(\pi) + [\ell - \mu(\pi)]^2} \le \varepsilon \\
    \Longleftrightarrow \quad &\sigma^2(\pi) \le \varepsilon[\sigma^2(\pi) + [\ell - \mu(\pi)]^2] \\
    \Longleftrightarrow \quad &\sigma^2(\pi) \le \varepsilon\sigma^2(\pi) + \varepsilon[\ell - \mu(\pi)]^2 \\
    \Longleftrightarrow \quad &(1 - \varepsilon)\sigma^2(\pi) \le \varepsilon[\ell - \mu(\pi)]^2 \\
    \Longleftrightarrow \quad &\frac{1 - \varepsilon}{\varepsilon}\sigma^2(\pi) \le [\ell - \mu(\pi)]^2 \\
    \Longleftrightarrow \quad &\left(\frac{1}{\varepsilon} - 1\right)\sigma^2(\pi) - [\ell - \mu(\pi)]^2 \le 0.
\end{align*}

In the boundary case $\mu(\pi) = \ell$, Cantelli's inequality no longer applies ($\lambda = 0$), so we verify the VaR constraint directly. Substituting $\mu(\pi) = \ell$ into the Cantelli constraint (\ref{eq:chebyshev_constraint}) eliminates its second term:
$$
B(\pi) = \left(\frac{1}{\varepsilon} - 1\right)\sigma^2(\pi) - [\ell - \ell]^2 = \left(\frac{1}{\varepsilon} - 1\right)\sigma^2(\pi) \le 0.
$$
Since $\varepsilon \in (0, 1)$, the factor $\frac{1}{\varepsilon} - 1$ is strictly positive, and $\sigma^2(\pi) \ge 0$ by definition; their product can only be non-positive if $\sigma^2(\pi) = 0$. A cost return with zero variance equals its mean with probability one, so $C(\tau) = \mu(\pi) = \ell$ almost surely and $\mathbb{P}(C(\tau) > \ell) = 0 \le \varepsilon$. Combining the two cases, $\mu(\pi) < \ell$ (via Cantelli's inequality) and $\mu(\pi) = \ell$ (via the zero-variance argument above): under the mean constraint $\mu(\pi) \le \ell$, the Cantelli constraint implies the VaR constraint $\mathbb{P}(C(\tau) > \ell) \le \varepsilon$.

\subsection{Derivation of Equation \ref{eq:square-cost}}
\label[appendix]{app:square-cost}

It is possible to break down the square cost return into the discounted sum of
local per-step terms:

\begin{gather*}
    \begin{split}
    C(\tau)^2 &= \bigg( \sum_{t=0}^\infty\gamma_c^t c_t \bigg)^2 \\
    &= \sum_{t=0}^\infty\gamma_c^t c_t \sum_{k=0}^\infty\gamma_c^k c_k\\
    &= \sum_{t=0}^\infty\gamma_c^t c_t \bigg( \sum_{k=0}^{t-1}\gamma_c^k c_k + \gamma_c^t c_t + \sum_{k=t+1}^\infty\gamma_c^k c_k \bigg)\\
    &= \sum_{t=0}^\infty\gamma_c^t c_t \bigg( y_t + \gamma_c^t c_t + \sum_{k=t+1}^\infty\gamma_c^k c_k \bigg)\\
    &= \sum_{t=0}^\infty\gamma_c^{2t}c_t^2 + \sum_{t=0}^\infty\gamma_c^t c_t y_t + \sum_{t=0}^\infty\gamma_c^t c_t \sum_{k=t+1}^\infty\gamma_c^k c_k\\
    &= \sum_{t=0}^\infty\gamma_c^{2t}c_t^2 + 2\sum_{t=0}^\infty\gamma_c^t c_t y_t\\
    &= \sum_{t=0}^\infty\gamma_c^t(\gamma_c^{t}c_t^2 + 2 c_t y_t).
\end{split}
\end{gather*}

Writing the square as a product of two copies of the sum, indexed by $t$ and
$k$, lets us hold the outer index $t$ fixed and split only the inner sum, into
the terms strictly before $t$, the term at $k = t$, and the terms strictly
after $t$. The first of these is by definition the accumulated cost
$y_t = \sum_{k=0}^{t-1}\gamma_c^k c_k$, which is where the recursion in
Equation \ref{eq:y-recursion} enters.

The only remaining step is the last cross term. Relabeling its two indices,
$(t, k) \mapsto (k, t)$, maps the region $k > t$ onto the region $k < t$ and
leaves the summand $\gamma_c^{k+t}c_kc_t$ unchanged, so
\begin{equation*}
    \sum_{t=0}^\infty\gamma_c^t c_t \sum_{k=t+1}^\infty\gamma_c^k c_k
    = \sum_{t=0}^\infty\gamma_c^t c_t \sum_{k=0}^{t-1}\gamma_c^k c_k
    = \sum_{t=0}^\infty\gamma_c^t c_t y_t ,
\end{equation*}
which is what makes the two cross terms combine into the factor of $2$. Each
unordered pair $\{k, t\}$ with $k \ne t$ is counted once looking forward from
$t$ and once looking back from $k$; $y_t$ collects only the backward half,
and the factor of $2$ accounts for the other.

\subsection{Derivation of Equation \ref{eq:cpo-cheb}}
\label[appendix]{app:cpo-cantelli}

We need to show that we can reconstruct the Cantelli-VaR constraint in a CMDP
form. First, we take the moments of the cost return to define the mean and
variance:

\begin{gather*}
    \mu(\pi) = \mathbb{E}_{\tau \sim \pi}\left[ C(\tau) \right] \\
    \tilde \mu(\pi) = \mathbb{E}_{\tau \sim \pi}\left[ C(\tau)^2 \right] \\
    \sigma^2(\pi) = \tilde \mu(\pi) - \mu(\pi)^2.
\end{gather*}

These can be combined with the square cost return in Equation
\ref{eq:square-cost} to break down the global episode-level Cantelli-VaR
constraint into a form containing a local per-step discounted term with the
augmented state space (\ref{eq:aug-state}):

\begin{gather*}
    \begin{split}
        B(\pi) &= \left(\frac{1}{\varepsilon} - 1\right) \sigma^2(\pi) - [\ell - \mu(\pi)]^2 \\
        &= \left(\frac{1}{\varepsilon} - 1\right) [\tilde{\mu}(\pi) - \mu(\pi)^2] - \ell^2 + 2 \ell \mu(\pi) - \mu(\pi)^2 \\
        &= \left(\frac{1}{\varepsilon} - 1\right) \tilde{\mu}(\pi) - \frac{1}{\varepsilon} \mu(\pi)^2 + 2 \ell \mu(\pi) - \ell^2 \\
        &= \left(\frac{1}{\varepsilon} - 1\right) \tilde{\mu}(\pi) + 2\ell \mu(\pi) - \left[\frac{1}{\varepsilon} \mu(\pi)^2 + \ell^2\right]
    \end{split}
\end{gather*}

This maps to the final constraint form in Equation \ref{eq:cpo-cheb} setting the policy-dependent bound $\rho(\pi) = \frac{1}{\varepsilon}\mu(\pi)^2 + \ell^2$:
\begin{equation*}
    B(\pi) = \left(\frac{1}{\varepsilon} - 1\right) \tilde{\mu}(\pi) + 2\ell \mu(\pi) - \rho(\pi) \le 0
\end{equation*}

The bundling of $\mu(\pi)^2$ into $\rho(\pi)$ is not merely notational. The
terms $\tilde{\mu}(\pi)$ and $\mu(\pi)$ are expected discounted sums of
per-step costs on the augmented state space, so they admit
advantage-based trust-region surrogates. In contrast, $\mu(\pi)^2$ is the
square of an expectation: it equals $\mathbb{E}_{\tau, \tau' \sim \pi}
[C(\tau) C(\tau')]$ for independent trajectories $\tau$ and $\tau'$, an
expectation over a pair of trajectories rather than one, so it cannot be
written as the expected return of any per-step cost, even after state
augmentation. We therefore move
$\mu(\pi)^2$, together with the constant $\ell^2$, to the constraint-limit
side as the policy-dependent bound $\rho(\pi)$. This term is approximated by
substituting the trust-region surrogate $L_\mu(\pi)$ (\cref{eq:lmu}) into the square rather than through
a trust-region surrogate of its own, which is the origin of the quadratic term
in Equation \ref{eq:d-surrogate} (\cref{app:d-surrogate}), and its
approximation error is bounded separately via a difference-of-squares
argument in \cref{app:theorem}.

\subsection{Derivation of Equation \ref{eq:d-surrogate}}
\label[appendix]{app:d-surrogate}

First, we can expand the difference between $\rho(\pi)$ and $\rho(\pi_k)$:

\begin{gather*}
    \rho(\pi) - \rho(\pi_k) = \frac{1}{\varepsilon} [\mu(\pi)^2 - \mu(\pi_k)^2]\\
    \rho(\pi) = \rho(\pi_k) + \frac{1}{\varepsilon} [\mu(\pi)^2 - \mu(\pi_k)^2].
\end{gather*}

Then, we can take the trust-region surrogate of $\mu(\pi)$ around some other
policy $\pi_k$:

\begin{gather*}
    \mu(\pi) \approx L_\mu(\pi) = \mu(\pi_k) + \frac{1}{1-\gamma_c} \mathbb{E}_{\substack{x \sim d^c_{\pi_k} \\ a \sim \pi}}[A^C_{\pi_k}(x, a)]\\
    L_\mu(\pi)^2 = \mu(\pi_k)^2 + \frac{2\mu(\pi_k)}{1 - \gamma_c}\mathbb{E}_{\substack{x \sim d^c_{\pi_k} \\ a \sim \pi}}[A^C_{\pi_k}(x, a)] + \frac{1}{(1 - \gamma_c)^2} \mathbb{E}_{\substack{x \sim d^c_{\pi_k} \\ a \sim \pi}}[A^C_{\pi_k}(x, a)]^2.
\end{gather*}

Substituting this in gets the trust-region surrogate $L_\rho(\pi)$:

\begin{align*}
    L_\rho(\pi) &= \rho(\pi_k) + \frac{1}{\varepsilon} [L_\mu(\pi)^2 - \mu(\pi_k)^2]\\
    &= \rho(\pi_k) + \frac{1}{\varepsilon(1-\gamma_c)} \left(2\mu(\pi_k)\mathbb{E}_{\substack{x \sim d^c_{\pi_k} \\ a \sim \pi}}[A^C_{\pi_k}(x, a)] + \frac{1}{1 - \gamma_c} \mathbb{E}_{\substack{x \sim d^c_{\pi_k} \\ a \sim \pi}}[A^C_{\pi_k}(x, a)]^2\right).
\end{align*}

\subsection{Derivation of Theorem \ref{th:worst-case-violation}}
\label{app:theorem}
Given a current policy $\pi_k$, we aim to show the worst-case constraint
violation of a policy $\pi_{k+1}$ following the Canary update rule
(\ref{eq:varcpo-constraint}). We start by considering the factored Cantelli-VaR
constraint $B(\pi)$:
\begin{equation*}
    B(\pi_{k+1}) = \left(\frac{1}{\varepsilon} - 1\right) \tilde{\mu}(\pi_{k+1}) + 2\ell \mu(\pi_{k+1}) - \rho(\pi_{k+1}).
\end{equation*}

Using the trust-region surrogates $L_{\tilde{\mu}}$, $L_{\mu}$, and
$L_\rho$, the algorithm update rule enforces the linearized constraint
$\left(\frac{1}{\varepsilon} - 1\right) L_{\tilde{\mu}}(\pi_{k+1}) + 2\ell L_{\mu}(\pi_{k+1}) -
L_\rho(\pi_{k+1}) \le 0$. Therefore, the true constraint violation is
bounded by the approximation error:

\begin{align*}
    B(\pi_{k+1}) &\le B(\pi_{k+1}) - \left[ \left(\frac{1}{\varepsilon} - 1\right) L_{\tilde{\mu}}(\pi_{k+1}) + 2\ell L_{\mu}(\pi_{k+1}) - L_\rho(\pi_{k+1}) \right] \\
    &= \left(\frac{1}{\varepsilon} - 1\right) \left[ \tilde{\mu}(\pi_{k+1}) - L_{\tilde{\mu}}(\pi_{k+1}) \right] + 2\ell \left[ \mu(\pi_{k+1}) - L_{\mu}(\pi_{k+1}) \right] + \left[ L_\rho(\pi_{k+1}) - \rho(\pi_{k+1}) \right].
\end{align*}
Applying the triangle inequality gives:
\begin{equation*}
    B(\pi_{k+1}) \le \left(\frac{1}{\varepsilon} - 1\right) |\tilde{\mu}(\pi_{k+1}) - L_{\tilde{\mu}}(\pi_{k+1})| + 2|\ell| \, |\mu(\pi_{k+1}) - L_{\mu}(\pi_{k+1})| + |L_\rho(\pi_{k+1}) - \rho(\pi_{k+1})|.
\end{equation*}

From standard CPO theory, the worst-case approximation errors for the linear
expected returns are bounded by:
\begin{gather*}
    |\tilde{\mu}(\pi_{k+1}) - L_{\tilde{\mu}}(\pi_{k+1})| \le \frac{\sqrt{2 \delta} \gamma_c}{(1 - \gamma_c)^2} \alpha_{\pi_{k+1}}^{\tilde{C}}\\
    |\mu(\pi_{k+1}) - L_{\mu}(\pi_{k+1})| \le \frac{\sqrt{2 \delta} \gamma_c}{(1 - \gamma_c)^2} \alpha_{\pi_{k+1}}^{C},
\end{gather*}
where $\alpha_{\pi_{k+1}}^{\tilde{C}} = \max_x |\mathbb{E}_{a\sim
\pi_{k+1}}[A_{\pi_k}^{\tilde{C}}(x, a)]|$ and $\alpha_{\pi_{k+1}}^{C} = \max_x
|\mathbb{E}_{a\sim \pi_{k+1}}[A_{\pi_k}^{C}(x, a)]|$.

We then need to deal with the approximation error in the constraint limit
$\rho(\pi_{k+1})$ itself. First we define variables $X$ and $Y$ for notational
brevity:
\begin{gather*}
    X = \mathbb{E}_{\substack{x \sim d^c_{\pi_{k+1}} \\ a \sim \pi_{k+1}}}[A_{\pi_k}^C(x, a)] = \sum_x d^c_{\pi_{k+1}}(x)\sum_a \pi_{k+1}(a\mid x) A_{\pi_k}^C(x, a)\\
    Y = \mathbb{E}_{\substack{x \sim d^c_{\pi_{k}} \\ a \sim \pi_{k+1}}}[A_{\pi_k}^C(x, a)] = \sum_x d^c_{\pi_k}(x) \sum_a \pi_{k+1}(a\mid x) A_{\pi_k}^C(x, a).
\end{gather*}

This allows us to write the error in the constraint limit itself as:
\begin{equation*}
    |L_\rho(\pi_{k+1}) - \rho(\pi_{k+1})| \le \frac{1}{\varepsilon(1-\gamma_c)}\left[2 |\mu(\pi_k)| \, |X - Y| + \frac{1}{1 - \gamma_c}|X^2 - Y^2|\right].
\end{equation*}

To bound this, we now need to find the limits for $|X-Y|$ and $|X^2 - Y^2|$.
Defining the total variational divergence of the policy update in discrete
action space as $D_{TV}(\pi_{k+1}, \pi_k) = \frac{1}{2} \sum_a |\pi_{k+1}(a \mid
x) - \pi_k(a \mid x)|$, we can use the existing result of the worst case state
visitation frequency difference:
\begin{equation*}
    || d^c_{\pi_{k+1}}(x) - d^c_{\pi_k}(x)||_1 \le \frac{2 \gamma_c}{1 - \gamma_c} \mathbb{E}_{x \sim d^c_{\pi_k}}[D_{TV}(\pi_{k+1}, \pi_k)],
\end{equation*}
and applying the same trust region bound theory in CPO:
\begin{equation*}
    |X-Y| \le \frac{2 \gamma_c \alpha_{\pi_{k+1}}^C}{1 - \gamma_c}\mathbb{E}_{x \sim d^c_{\pi_k}}[D_{TV}(\pi_{k+1}, \pi_k)].
\end{equation*}
To handle $|X^2 - Y^2|$, we can first decompose it:
\begin{equation*}
    |X^2 - Y^2| = |X-Y||X+Y|.
\end{equation*}
This is useful since we previously defined the bound for $|X-Y|$ already, so we
are now left with $|X+Y|$:
\begin{gather*}
    |X+Y| \le |X| + |Y|.
\end{gather*}
First, we can define an inner term of $X$ and $Y$ as a function of the augmented
state $x$ alone:
\begin{gather*}
    f(x) = \mathbb{E}_{a \sim \pi_{k+1}} [A_{\pi_k}^C(x, a)] \\
    \alpha_{\pi_{k+1}}^C = \max_{x} |f(x)|.
\end{gather*}
Then we can derive the bound for $|X|$:
\begin{align*}
    |X| &\le \sum_{x} d_{\pi_{k+1}}^c(x) \cdot |f(x)| \\
    &\le \sum_{x} d_{\pi_{k+1}}^c(x) \cdot \alpha_{\pi_{k+1}}^C\\
    &\le \alpha_{\pi_{k+1}}^C \left( \sum_{x} d_{\pi_{k+1}}^c(x) \right)\\
    &\le \alpha_{\pi_{k+1}}^C.
\end{align*}
The same steps apply for $|Y|$:
\begin{equation*}
    |Y| \le \alpha_{\pi_{k+1}}^C
\end{equation*}
such that the bound for $|X+Y|$ can be derived:
\begin{equation*}
    |X+Y| \le 2 \alpha_{\pi_{k+1}}^C
\end{equation*}
and $|X^2 - Y^2|$:
\begin{equation*}
    |X^2 - Y^2| \le \frac{4 \gamma_c (\alpha_{\pi_{k+1}}^C)^2}{1 - \gamma_c} \mathbb{E}_{x \sim d^c_{\pi_k}}[D_{TV}(\pi_{k+1}, \pi_k)].
\end{equation*}

We can also use Pinsker's inequality to bound the total variational distance
with the KL divergence, which in turn is bounded in the update step by $\delta$:
\begin{equation*}
    D_{TV}(P, Q) \le \sqrt{\frac{1}{2} D_{KL}(P, Q)},
\end{equation*}
and putting everything together, we can create the final worst-case constraint
violation by summing the three separated bound components:
\begin{equation*}
    B(\pi_{k+1}) \le \frac{\sqrt{2\delta}\gamma_c}{(1-\gamma_c)^2} \left( \left[\frac{1}{\varepsilon} - 1\right] \alpha_{\pi_{k+1}}^{\tilde{C}} + 2 |\ell| \alpha_{\pi_{k+1}}^{C} + \frac{2\alpha_{\pi_{k+1}}^{C}}{\varepsilon} \left[ |\mu(\pi_k)| + \frac{\alpha_{\pi_{k+1}}^{C}}{1-\gamma_c} \right] \right).
\end{equation*}

The mean overshoot bound (\ref{eq:mean-overshoot}) follows directly from the
linearized mean constraint of OP \ref{eq:varcpo-constraint} together with the approximation
error bound above: $L_\mu(\pi_{k+1}) \le \ell$, and $|\mu(\pi_{k+1}) -
L_\mu(\pi_{k+1})| \le \frac{\sqrt{2\delta}\gamma_c}{(1-\gamma_c)^2}
\alpha_{\pi_{k+1}}^{C}$, so
\begin{equation*}
    \mu(\pi_{k+1}) \le L_\mu(\pi_{k+1}) + |\mu(\pi_{k+1}) -
L_\mu(\pi_{k+1})| \le \ell + \frac{\sqrt{2\delta}\gamma_c}{(1-\gamma_c)^2} \alpha_{\pi_{k+1}}^{C}.
\end{equation*}

\subsection{Proof of Corollary \ref{cor:recovery}}
\label[appendix]{app:corollary}

We first show that the VaR constraint is provably violated on $V$. Define the reflected cost return $X = -C(\tau)$, with
$$
\mu_X(\pi) = -\mu(\pi), \qquad
\sigma_X^2(\pi) = \sigma^2(\pi), \qquad
\ell_X = -\ell.
$$
Reflection exchanges the two tails,
$$
\{X \ge \ell_X\} = \{-C(\tau) \ge -\ell\} = \{C(\tau) \le \ell\},
$$
so an upper bound on the upper tail of $X$ is an upper bound on the lower tail
of $C(\tau)$. Mirroring the construction of \cref{app:cantelli-var}, which sets
$\lambda = \ell - \mu(\pi)$, we take
$$
\lambda_X = \ell_X - \mu_X(\pi) = \mu(\pi) - \ell,
$$
which is positive exactly when $\mu(\pi) > \ell$. Cantelli's inequality
(\ref{eq:cantelli}) applied to $X$ then gives
$$
\mathbb{P}\big(X - \mu_X(\pi) \ge \lambda_X\big)
\le \frac{\sigma_X^2(\pi)}{\sigma_X^2(\pi) + \lambda_X^2},
$$
and substituting the definitions back,
$$
\mathbb{P}\big(C(\tau) \le \ell\big)
\le \frac{\sigma^2(\pi)}{\sigma^2(\pi) + [\mu(\pi) - \ell]^2}.
$$
Since $\mathbb{P}(C(\tau) > \ell) = 1 -
\mathbb{P}(C(\tau) \le \ell)$, taking complements gives
$$
\mathbb{P}\big(C(\tau) > \ell\big)
\ge \frac{[\mu(\pi) - \ell]^2}{\sigma^2(\pi) + [\mu(\pi) - \ell]^2}.
$$
Requiring this lower bound to exceed $\varepsilon$, and using $\varepsilon \in
(0, 1)$ and $\lambda_X > 0$ throughout,
\begin{align*}
    \frac{\lambda_X^2}{\sigma^2(\pi) + \lambda_X^2} > \varepsilon
    \quad &\Longleftrightarrow \quad \lambda_X^2 (1 - \varepsilon) > \varepsilon \sigma^2(\pi) \\
    \quad &\Longleftrightarrow \quad \lambda_X > \sigma(\pi) \sqrt{\frac{\varepsilon}{1 - \varepsilon}}
    = \frac{\sigma(\pi)}{\sqrt{\frac{1}{\varepsilon} - 1}},
\end{align*}
that is, $\mu(\pi) > \ell + \sigma(\pi) / \sqrt{\frac{1}{\varepsilon} - 1}$, which defines the
provably violating region $V$.

We now bound the expected cost return of each tier. All three tiers of the selection rule in \cref{sec:recovery} constrain
$\pi_{k+1}$ to the trust region $\bar{D}_{KL}(\pi_{k+1}, \pi_k) \le \delta$,
so the CPO approximation-error bound used in \cref{app:theorem},
$|\mu(\pi_{k+1}) - L_\mu(\pi_{k+1})| \le K \alpha^C_{\pi_{k+1}}$, applies to
every tier. In particular,
\begin{equation}
\label{eq:mu-surrogate-error}
    \mu(\pi_{k+1}) \le L_\mu(\pi_{k+1}) + |\mu(\pi_{k+1}) - L_\mu(\pi_{k+1})| \le L_\mu(\pi_{k+1}) + K \alpha^C_{\pi_{k+1}},
\end{equation}
and it suffices to bound the surrogate $L_\mu(\pi_{k+1})$ in each tier.

\emph{Tiers a--b} ($m_\mu \le \ell$). Both OP \ref{eq:varcpo-constraint} and
Tier b of (\ref{eq:recovery}) contain the constraint $L_\mu(\pi_{k+1}) \le \ell$, so
(\ref{eq:mu-surrogate-error}) gives $\mu(\pi_{k+1}) \le \ell + K
\alpha^C_{\pi_{k+1}}$ directly. For Tier a, this is the expected cost return bound
(\ref{eq:mean-overshoot}) of \cref{th:worst-case-violation}. Membership of $V$ requires $\mu(\pi_{k+1}) > \ell +
\sigma(\pi_{k+1})/\sqrt{\frac{1}{\varepsilon} - 1}$. If $\sigma(\pi_{k+1}) \ge \sqrt{\frac{1}{\varepsilon} - 1}\, K
\alpha^C_{\pi_{k+1}}$, then
$$
\ell + \frac{\sigma(\pi_{k+1})}{\sqrt{\frac{1}{\varepsilon} - 1}} \ge \ell + K \alpha^C_{\pi_{k+1}} \ge \mu(\pi_{k+1}),
$$
so $\pi_{k+1} \notin V$.

\emph{Tier c} ($m_\mu > \ell$). The current policy is always admissible, with
$L_\mu(\pi_k) = \mu(\pi_k)$ by the first-order validity of the surrogate
(\cref{sec:CPO}), so the Tier c minimizer in (\ref{eq:recovery}) satisfies
$L_\mu(\pi_{k+1}) = m_\mu \le \mu(\pi_k)$. Combining with
(\ref{eq:mu-surrogate-error}),
$$
\mu(\pi_{k+1}) \le \mu(\pi_k) + K \alpha^C_{\pi_{k+1}}.
$$

\section{Relative Variance of VaR Constraint Estimators}
\label[appendix]{app:scaling}

As established in \cref{sec:relative-variance}, a constrained policy update
depends on estimates of both the constraint value and its policy gradient. We
study the statistical quality of these estimates through their relative
variance (inverse signal-to-noise ratio), defined in
\cref{eq:relative-variance}. The purpose of this appendix is to characterize
how the relative variance of the indicator, CVaR, and Cantelli estimators
depends on the fraction of a trajectory batch that is informative for each
estimator.

Throughout, as in \cref{sec:relative-variance}, costs are non-negative,
$c_t \ge 0$ almost surely, so $C(\tau) \ge 0$, and we consider a batch of
$N$ independent trajectories
\[
    \tau_1,\ldots,\tau_N
    \overset{\mathrm{i.i.d.}}{\sim}
    \pi_\theta.
\]
Independence is assumed only across trajectories; states and actions
within an individual trajectory retain the temporal dependence induced
by the MDP.

We write each method's constraint in the canonical form
\[
    J_C(\theta) \leq 0,
\]
where $J_C(\theta)$ is the population constraint value and
$\hat{J}_C(\theta)$ its Monte Carlo estimator. Similarly,
$\nabla_\theta J_C(\theta)$ denotes the population constraint gradient
and $\hat{g}$ its Monte Carlo estimator. These are the quantities
consumed by a constrained optimizer, for example in a Lagrangian
multiplier update or a trust-region feasibility test.

For a quantity $m$ with estimator $\hat{m}$, recall from
\cref{eq:relative-variance} that its relative variance is
\[
    \kappa^2[m]
    =
    \frac{\operatorname{Var}[\hat{m}]}
         {\lVert m\rVert_2^2}.
\]
Relative variance is therefore undefined when the underlying signal
vanishes. In particular, in our case, when the quantity in question is $J_C(\theta)$, the relative variance of a constraint-value
estimator diverges mechanically at the constraint boundary
$J_C(\theta)=0$, independently of the estimator being considered. This is
an artifact of the metric rather than of any particular estimator, and
\cref{ass:slack} quantifies the distance from the boundary at which
the constraint values are compared.

\paragraph{Effective sample size.}
The three estimators differ fundamentally in the fraction of a
trajectory batch that enters their trajectory-level cost statistic.
To make this precise, let $h(\tau)$ denote the trajectory-level
statistic associated with an estimator, and define its \emph{active event} by
\[
    \mathcal{E}_h
    :=
    \{\tau : h(\tau) \neq 0\}.
\]
For a realized batch, define the \emph{effective sample size} associated with
$h$ as
\[
    N_{\mathrm{eff}}(h)
    :=
    \sum_{i=1}^{N}
    \mathbf{1}\{\tau_i \in \mathcal{E}_h\}
    =
    \sum_{i=1}^{N}
    \mathbf{1}\{h(\tau_i)\neq 0\}.
\]
Thus, $N_{\mathrm{eff}}(h)$ is a random variable. Since the trajectories
are identically distributed,
\[
    \mathbb{E}[N_{\mathrm{eff}}(h)]
    =
    N\,\mathbb{P}(\mathcal{E}_h).
\]
We use this expected effective sample size to characterize the
sampling regime of each estimator.

Let
\[
    p
    :=
    \mathbb{P}\bigl(C(\tau)>\ell\bigr)
\]
denote the policy's violation probability, and let
\[
    f
    :=
    \mathbb{P}\bigl(C(\tau)>0\bigr)
\]
denote the probability that a trajectory incurs any positive cost.
Since \cref{ass:support} requires $\ell>0$,
\[
    \{C(\tau)>\ell\}
    \subseteq
    \{C(\tau)>0\},
\]
and hence
\[
    p \leq f.
\]

For the \emph{indicator estimator}, the trajectory-level statistic is
\[
    h_{\mathrm{Ind}}(\tau)
    :=
    \mathbf{1}\{C(\tau)>\ell\}.
\]
Its active event is therefore
\[
    \mathcal{E}_{\mathrm{Ind}}
    =
    \{C(\tau)>\ell\},
\]
which occurs with probability $p$. Consequently,
\[
    \mathbb{E}
    \left[
        N_{\mathrm{eff}}(h_{\mathrm{Ind}})
    \right]
    =
    pN.
\]
Since $p \le f$, the indicator estimator's expected effective sample size
$pN$ is at most $fN$, and shrinks as the threshold tightens while $f$ does
not.

For the \emph{CVaR estimator}, the trajectory-level statistic is
\[
    h_{\mathrm{CVaR}}(\tau)
    :=
    \bigl(C(\tau)-\eta^*\bigr)^+.
\]
Its active event is
\[
    \mathcal{E}_{\mathrm{CVaR}}
    =
    \{C(\tau)>\eta^*\}.
\]
Under \cref{ass:support}, which requires that the cost return has no atom at $\eta^*$,
\[
    \mathbb{P}\bigl(C(\tau)>\eta^*\bigr)
    =
    \varepsilon,
\]
and hence
\[
    \mathbb{E}
    \left[
        N_{\mathrm{eff}}(h_{\mathrm{CVaR}})
    \right]
    =
    \varepsilon N.
\]
Akin to the indicator, this is thresholded by $\varepsilon$ rather than $f$:
since \cref{ass:support} requires $\eta^* > 0$,
\[
    \{C(\tau)>\eta^*\}
    \subseteq
    \{C(\tau)>0\},
\]
so $\varepsilon \le f$. Furthermore, the inequality is strict, because $\eta^* > 0$
is the infimum in \cref{eq:strict-var}, so $\eta = 0$ does not satisfy
$\mathbb{P}(C(\tau) > \eta) \le \varepsilon$, that is, $\varepsilon < f$. The
CVaR estimator's expected effective sample size $\varepsilon N$ is therefore
strictly below $fN$, and shrinks as the threshold tightens while $f$ does
not.

The \emph{Cantelli estimator} does not threshold trajectories at either
$\ell$ or $\eta^*$. Instead, it estimates the first two cost moments
using the trajectory-level statistics
\[
    h_1(\tau) := C(\tau),
    \qquad
    h_2(\tau) := C(\tau)^2.
\]
Since costs are non-negative, both statistics have the same active
event,
\[
    \mathcal{E}_{\mathrm{Cantelli}}
    =
    \{C(\tau)>0\},
\]
which occurs with probability $f$. Therefore,
\[
    \mathbb{E}
    \left[
        N_{\mathrm{eff}}(h_1)
    \right]
    =
    \mathbb{E}
    \left[
        N_{\mathrm{eff}}(h_2)
    \right]
    =
    fN.
\]

Thus, the expected effective sample sizes associated with the three
methods are
\[
    \begin{array}{c|c|c}
        \text{Estimator}
        &
        \text{Active event}
        &
        \mathbb{E}[N_{\mathrm{eff}}]
        \\ \hline
        \text{Indicator}
        &
        \{C(\tau)>\ell\}
        &
        pN
        \\
        \text{CVaR}
        &
        \{C(\tau)>\eta^*\}
        &
        \varepsilon N
        \\
        \text{Cantelli}
        &
        \{C(\tau)>0\}
        &
        fN
    \end{array}.
\]

These expected effective sample sizes provide the intuition for the
relative-variance scaling derived below. Under the corresponding
non-degeneracy assumptions, the derivations show that the
relative variance scales, up to constant factors, as the inverse
expected effective sample size,
\[
    \kappa^2
    =
    \Theta\left(
        \frac{1}{\mathbb{E}[N_{\mathrm{eff}}]}
    \right).
\]
This gives the scaling classes
\[
    \Theta\left(\frac{1}{pN}\right),
    \qquad
    \Theta\left(\frac{1}{\varepsilon N}\right),
    \qquad
    \Theta\left(\frac{1}{fN}\right)
\]
for the indicator, CVaR, and Cantelli estimators, respectively. The
following subsections establish the corresponding results and make
explicit the conditions under which these scaling classes hold.

\paragraph{Dense and sparse cost regimes.}
The distinction between the three expected effective sample sizes is
most important in the dense-cost, low-violation regime. If $f$ remains
bounded away from zero while $p,\varepsilon\rightarrow 0$, then the
expected fractions of the batch active for the indicator and CVaR
estimators satisfy
\[
    \frac{
        \mathbb{E}[N_{\mathrm{eff}}(h_{\mathrm{Ind}})]
    }{N}
    =
    p
    \rightarrow 0,
\]
and
\[
    \frac{
        \mathbb{E}[N_{\mathrm{eff}}(h_{\mathrm{CVaR}})]
    }{N}
    =
    \varepsilon
    \rightarrow 0,
\]
whereas for either Cantelli moment,
\[
    \frac{
        \mathbb{E}[N_{\mathrm{eff}}(h_i)]
    }{N}
    =
    f,
    \qquad
    i\in\{1,2\},
\]
remains bounded away from zero. Thus, indicator and CVaR estimation
become increasingly dominated by rare-event sampling, while the
Cantelli moment estimators continue to extract information from a
non-vanishing fraction of the batch.

This distinction disappears when the environmental cost signal itself
is sparse. If $f\leq\varepsilon$, the CVaR quantile satisfies
$\eta^*=0$. In this regime the CVaR active event coincides with the
positive-cost support, and its expected effective sample size becomes
$fN$, matching that of the Cantelli moment estimators. The effective
sample size advantage considered in this appendix is therefore
specific to dense-cost regimes, which is the requirement ($f > \varepsilon$)
noted in \cref{sec:relative-variance} and stated as a limitation in
\cref{sec:conclusion}.

\paragraph{From trajectory statistics to constraint values.}
The expected effective sample size alone does not determine relative
variance: the distribution of the active contributions may itself change
with $p$, $\varepsilon$, or $f$, and the constraint value is not the
statistic $h(\tau)$ itself but a quantity built from it. The object a
constrained optimizer consumes is the constraint value $J_C(\theta)$ and its
policy gradient, so we now state the two results the per-method derivations
instantiate in terms of $J_C(\theta)$ directly, and make explicit where each
assumption of \cref{sec:relative-variance} enters.

Every constraint value below decomposes as
\begin{equation}
\label{eq:jc-decomp}
    J_C(\theta) = \mathbb{E}_{\tau \sim \pi_\theta}[x(\tau)] + b,
\end{equation}
where $x(\tau)$ is a \emph{sampled statistic} of a single trajectory and $b$
is an \emph{offset} that does not depend on the sampled trajectories
$\tau_1, \ldots, \tau_N$, such as the threshold $-\varepsilon$ in the
indicator constraint or the cost limit $-\ell$ in the mean constraint. The sampled statistic is built from the method's
trajectory-level statistic $h(\tau)$: 
\begin{gather}
    x_{\mathrm{Ind}}(\tau) = h_{\mathrm{Ind}}(\tau)\\
    x_{\mathrm{CVaR}}(\tau) = \frac{h_{\mathrm{CVaR}}(\tau)}{\varepsilon}\\
    x_{\mathrm{Cantelli}}(\tau) = c_1 h_1(\tau) + c_2 h_2(\tau) \quad \text{(\cref{app:cantelli})}.
\end{gather}
In every case, $x(\tau) = 0$ whenever $h(\tau) = 0$, so
$x(\tau)$ vanishes outside the active event $\mathcal{E}_h$ and has the same
active probability
\begin{equation}
    q := \mathbb{P}(\mathcal{E}_h) \in \{p, \varepsilon, f\}.
\end{equation}
Let
\begin{equation}
    m_a := \mathbb{E}[x(\tau) \mid \mathcal{E}_h], \qquad
    S_a := \mathbb{E}[x(\tau)^2 \mid \mathcal{E}_h], \qquad
    R_x := \frac{S_a}{m_a^2}
\end{equation}
denote the conditional first and second moments of the sampled statistic
and their \emph{conditional ratio}, the relative second moment of a single
active contribution. By the Law of Total Expectation, since $x(\tau)$
vanishes off $\mathcal{E}_h$,
\begin{equation}
\label{eq:jc-signal}
    J_C(\theta) = q m_a + b.
\end{equation}
Define the \emph{signal share}
\begin{equation}
\label{eq:signal-share}
    r := \frac{q |m_a|}{|J_C(\theta)|}.
\end{equation}
\Cref{ass:slack} bounds $r$ above and
below.

For the policy gradient, define the trajectory score
\begin{equation}
    s(\tau) := \sum_t \nabla_\theta \log \pi_\theta(a_t \mid s_t).
\end{equation}
The score-function identity \cite{reinforce} applied to \cref{eq:jc-decomp}
gives,
\begin{equation}
\label{eq:grad-decomp}
    \nabla_\theta J_C(\theta) = \mathbb{E}_{\tau \sim \pi_\theta}[s(\tau) x(\tau)],
\end{equation}
so the gradient is the expectation of the sampled statistic $s(\tau)
x(\tau)$, which vanishes off the same active event $\mathcal{E}_h$ (whenever
$x(\tau) = 0$, the gradient contribution is necessarily zero; we define the
active event through $x(\tau)$ rather than $s(\tau) x(\tau)$ to isolate the
sparsity induced by the cost statistic from incidental zeros of the score).
Let
\begin{equation}
    m_{g} := \mathbb{E}[s(\tau) x(\tau) \mid \mathcal{E}_h], \qquad
    S_{g} := \mathbb{E}[\|s(\tau) x(\tau)\|^2 \mid \mathcal{E}_h], \qquad
    R_{g} := \frac{S_{g}}{\|m_{g}\|^2}
\end{equation}
denote the corresponding conditional moments and ratio for the gradient, so
that $\nabla_\theta J_C(\theta) = q m_g$. The gradient has no offset, and
therefore no signal share.

The conditional ratios $R_x$ and $R_g$ are upper bounded by \cref{ass:snr}
for the indicator and CVaR constraints since $x(\tau)$ is a scalar multiple of $h(\tau)$. For the Cantelli constraint, $x(\tau)$ is
a linear combination of two statistics of \cref{ass:snr}, and bounding
$R_x$ and $R_g$ additionally requires \cref{ass:no-cancel}
(\cref{app:cantelli}).

\begin{lemma}[Relative variance of a constraint value]
\label[lemma]{lem:value}
Let $J_C(\theta)$ decompose as in \cref{eq:jc-decomp}, with the Monte Carlo
estimator $\hat{J}_C = \frac{1}{N}\sum_{i=1}^N x(\tau_i) + b$ over $N$
independent trajectories. Then
\begin{equation}
\label{eq:value-exact}
    \kappa^2[J_C(\theta)] = r^2 \left( \frac{R_x}{qN} - \frac{1}{N} \right),
\end{equation}
and
\begin{equation}
\label{eq:value-sandwich}
    \frac{(1 - q)\, r^2}{qN} \;\le\; \kappa^2[J_C(\theta)] \;\le\; \frac{r^2 R_x}{qN}.
\end{equation}
If $R_x \le K_x$ for a constant $K_x$ independent of $(p, \varepsilon, f)$,
and \cref{ass:slack} holds, then for $q$ bounded away from one,
\begin{equation}
\label{eq:value-theta}
    \kappa^2[J_C(\theta)] = \Theta\left(\frac{1}{qN}\right).
\end{equation}
\end{lemma}
\begin{proof}
The estimator is unbiased: $\mathbb{E}[\hat{J}_C] = \mathbb{E}[x(\tau)] + b
= q m_a + b = J_C(\theta)$ by \cref{eq:jc-signal}. Starting from the
definition of relative variance (\ref{eq:relative-variance}),
\begin{equation}
\label{eq:zero-inflated}
\begin{split}
    \kappa^2[J_C(\theta)]
    &= \frac{\operatorname{Var}(\hat{J}_C)}{J_C(\theta)^2} \\
    &= \frac{1}{J_C(\theta)^2} \operatorname{Var}\left( \frac{1}{N} \sum_{i=1}^N x(\tau_i) \right)
    \qquad \text{($b$ does not depend on the batch)} \\
    &= \frac{\operatorname{Var}(x(\tau))}{N\, J_C(\theta)^2}
    \qquad \text{(independence across trajectories)} \\
    &= \frac{\mathbb{E}[x(\tau)^2] - \mathbb{E}[x(\tau)]^2}{N\, J_C(\theta)^2} \\
    &= \frac{q S_a - q^2 m_a^2}{N\, J_C(\theta)^2}
    \qquad \text{(Law of Total Expectation; $x(\tau) = 0$ off $\mathcal{E}_h$)} \\
    &= \frac{(q m_a)^2}{J_C(\theta)^2} \cdot \frac{q S_a - q^2 m_a^2}{N (q m_a)^2} \\
    &= r^2 \left( \frac{R_x}{qN} - \frac{1}{N} \right)
    \qquad \text{(\cref{eq:signal-share}; $R_x = S_a / m_a^2$)},
\end{split}
\end{equation}
which is \cref{eq:value-exact}. From it, the two sides of
\cref{eq:value-sandwich} follow as
\begin{equation}
\begin{split}
    \kappa^2[J_C(\theta)]
    &\ge r^2 \left( \frac{1}{qN} - \frac{1}{N} \right) = \frac{(1 - q)\, r^2}{qN}
    \qquad \text{(Jensen: $S_a \ge m_a^2$, so $R_x \ge 1$)}, \\
    \kappa^2[J_C(\theta)]
    &\le \frac{r^2 R_x}{qN}
    \qquad \text{(drop the non-negative $r^2 / N$)} \\
    &\le \frac{r^2 K_x}{qN}
    \qquad \text{($R_x \le K_x$; \cref{ass:snr} and \cref{ass:no-cancel})}.
\end{split}
\end{equation}
The lower bound uses no assumption. Under \cref{ass:slack}, $r_{\min} \le r
\le r_{\max}$, so for $q$ bounded away from one,
\begin{equation}
    \frac{(1 - q)\, r_{\min}^2}{qN} \le \kappa^2[J_C(\theta)] \le \frac{r_{\max}^2 K_x}{qN}
    \qquad \text{(\cref{ass:slack})},
\end{equation}
which is \cref{eq:value-theta}.
\end{proof}

The two sides of \cref{eq:value-sandwich} have different standing. The lower
bound holds for every policy with $r > 0$: whenever the sampled part of the
constraint is comparable to the constraint value, an expected effective sample
size $qN$ below one puts the estimator below unit signal-to-noise ratio. The
upper bound is where the ratio bound $R_x \le K_x$ enters, as the statement
that the relative second moment of a single active contribution stays
bounded while the active fraction shrinks. For the indicator and CVaR
constraints, $x(\tau)$ is a scalar multiple of $h(\tau)$ and \cref{ass:snr}
supplies it directly; for the Cantelli constraint, $x(\tau)$ is a linear
combination of two statistics of \cref{ass:snr}, and bounding $R_x$ and
$R_g$ additionally requires \cref{ass:no-cancel} (\cref{app:cantelli}).
\Cref{ass:slack} is needed on both sides, and
for a reason specific to constraint values: the upper bound
$\mathcal{O}(1/(qN))$ can be loose. If the sampled part $q m_a$ vanishes
while the offset $b$ holds the signal fixed, then $r \to 0$ and
\cref{eq:value-exact} shows the relative variance vanishes rather than
diverges; the CVaR constraint value exhibits exactly this behavior for
almost surely bounded cost returns (\cref{app:cvar}). \Cref{ass:slack} also subsumes the
boundary margin: $r \le r_{\max}$ implies $|J_C(\theta)| \ge q|m_a|/r_{\max}
> 0$.

\begin{lemma}[Relative variance of a constraint gradient]
\label[lemma]{lem:grad}
Let $J_C(\theta)$ decompose as in \cref{eq:jc-decomp} and satisfy the
gradient identity \cref{eq:grad-decomp}, $\nabla_\theta J_C(\theta) =
\mathbb{E}_{\tau \sim \pi_\theta}[s(\tau) x(\tau)]$, which holds whenever
the offset $b$ is independent of $\theta$, with the Monte Carlo gradient
estimator $\hat{g} = \frac{1}{N}\sum_{i=1}^N s(\tau_i) x(\tau_i)$ over $N$
independent trajectories. Then
\begin{equation}
\label{eq:grad-exact}
    \kappa^2[\nabla_\theta J_C(\theta)] = \frac{R_g}{qN} - \frac{1}{N},
\end{equation}
and
\begin{equation}
\label{eq:grad-sandwich}
    \frac{1 - q}{qN} \;\le\; \kappa^2[\nabla_\theta J_C(\theta)] \;\le\; \frac{R_g}{qN}.
\end{equation}
If $R_g \le K_g$ for a constant $K_g$ independent of $(p, \varepsilon, f)$,
then for $q$ bounded away from one,
\begin{equation}
\label{eq:grad-theta}
    \kappa^2[\nabla_\theta J_C(\theta)] = \Theta\left(\frac{1}{qN}\right),
\end{equation}
with constants $1 - q$ and $K_g$. The lower bound in \cref{eq:grad-sandwich}
holds for every policy.
\end{lemma}
\begin{proof}
By the gradient identity, $\hat{g}$ is the empirical mean of $Y(\tau) :=
s(\tau) x(\tau)$ and is unbiased for $\nabla_\theta J_C(\theta) = q m_g$.
Starting from the definition of relative variance
(\ref{eq:relative-variance}),
\begin{equation}
\begin{split}
    \kappa^2[\nabla_\theta J_C(\theta)]
    &= \frac{\operatorname{Var}(\hat{g})}{\|\nabla_\theta J_C(\theta)\|^2} \\
    &= \frac{\operatorname{Var}(Y(\tau))}{N\, \|\nabla_\theta J_C(\theta)\|^2}
    \qquad \text{(independence across trajectories)} \\
    &= \frac{\mathbb{E}[\|Y(\tau)\|^2] - \|\mathbb{E}[Y(\tau)]\|^2}{N\, \|\nabla_\theta J_C(\theta)\|^2} \\
    &= \frac{q S_g - q^2 \|m_g\|^2}{N\, q^2 \|m_g\|^2}
    \qquad \text{(Law of Total Expectation; $Y(\tau) = 0$ off $\mathcal{E}_h$; $\nabla_\theta J_C(\theta) = q m_g$)} \\
    &= \frac{R_g}{qN} - \frac{1}{N}
    \qquad \text{($R_g = S_g / \|m_g\|^2$)},
\end{split}
\end{equation}
which is \cref{eq:grad-exact}. There is no offset and hence no signal
share. The two sides of
\cref{eq:grad-sandwich} follow as
\begin{equation}
\begin{split}
    \kappa^2[\nabla_\theta J_C(\theta)]
    &\ge \frac{1}{qN} - \frac{1}{N} = \frac{1 - q}{qN}
    \qquad \text{(Jensen for $y \mapsto \|y\|^2$: $S_g \ge \|m_g\|^2$, so $R_g \ge 1$)}, \\
    \kappa^2[\nabla_\theta J_C(\theta)]
    &\le \frac{R_g}{qN}
    \qquad \text{(drop the non-negative $1 / N$)} \\
    &\le \frac{K_g}{qN}
    \qquad \text{($R_g \le K_g$; \cref{ass:snr} and \cref{ass:no-cancel})},
\end{split}
\end{equation}
which for $q$ bounded away from one is \cref{eq:grad-theta}. The lower bound
uses no assumption. A scalar factor multiplying both $x(\tau)$ and its
estimator, such as the $1/\varepsilon$ of the CVaR constraint, cancels in
the relative variance.
\end{proof}

The per-method derivations that follow therefore reduce to three steps:
identify the sampled statistic $x(\tau)$, its active probability $q$, and
its offset $b$; bound the conditional ratios $R_x$ and $R_g$, by
\cref{ass:snr} directly or, for the Cantelli constraint, by
\cref{ass:snr,ass:no-cancel} together; and, for the constraint value, read
off the signal share $r$ and state what \cref{ass:slack} requires of the
policy. \Cref{lem:value,lem:grad} then supply the rates.
We now derive the relative-variance scaling explicitly, beginning with
the indicator estimator.

\subsection{Indicator VaR: $N_\mathrm{eff}=pN$}
\label[appendix]{app:indicator}

The indicator constraint (\ref{eq:indicator}) bounds the expectation of the
Bernoulli indicator, $\mu_I(\theta) = \mathbb{E}_{\tau\sim
\pi_\theta}[I(\tau)] = \mathbb{P}(C(\tau) > \ell) = p$, by the violation
threshold. The VaR constraint is therefore:
\begin{equation}
    J_C(\theta) = \mu_I(\theta) - \varepsilon \le 0,
\end{equation}
with Monte Carlo estimator
\begin{equation}
    \hat{J}_C = \frac{1}{N} \sum_{i=1}^N I(\tau_i) - \varepsilon.
\end{equation}
This is of the form \cref{eq:jc-decomp} with sampled statistic $x(\tau) =
I(\tau)$, active probability $q = p$, and offset $b = -\varepsilon$, so
\cref{lem:value} applies. Since the indicator is idempotent ($I(\tau)^2 =
I(\tau)$), its conditional moments are $m_a = S_a = 1$, so $R_x = 1$ and
the ratio bound holds with $K_x = 1$. The signal is $J_C(\theta) = p -
\varepsilon$ (\ref{eq:jc-signal}). Any feasible policy has $p \le
\varepsilon$, so we write $p = c\varepsilon$ with $c \in (0, 1)$
parameterizing the fraction of the threshold the policy uses; the slack is
then $\varepsilon - p = \frac{1-c}{c} p$ and the signal share
(\ref{eq:signal-share}) is $r = p / (\varepsilon - p) = c/(1-c)$. Holding
$c$ fixed is \cref{ass:slack} for the indicator constraint. Substituting these
into \cref{lem:value}, the relative variance is exactly
\begin{equation}
    \kappa^2[J_C(\theta)] = \left(\frac{c}{1 - c}\right)^2 \left( \frac{1}{pN} - \frac{1}{N} \right)
    = \left(\frac{c}{1 - c}\right)^2 \frac{1 - p}{pN},
\end{equation}
which, since $R_x = 1$, sits on the lower edge of the bounds of
\cref{lem:value},
\begin{equation}
    \frac{(1 - p)}{pN} \left(\frac{c}{1 - c}\right)^2 \le \kappa^2[J_C(\theta)] \le \frac{1}{pN} \left(\frac{c}{1 - c}\right)^2,
\end{equation}
and hence, for fixed $c$ and with $1 - \varepsilon \le 1 - p \le 1$,
\begin{equation}
    \kappa^2[J_C(\theta)] = \Theta\left(\frac{1}{p N}\right).
\end{equation} Because $p =
c\varepsilon$ ties the violation probability to the threshold, this rate
diverges as the violation threshold becomes strict ($\varepsilon \to 0$): the
constraint value cannot be evaluated with bounded relative error, and, the
expression being exact, no choice of estimator constant can prevent it.

For the gradient, \cref{lem:grad} applies with the same $x(\tau)$ and $q =
p$: the conditional gradient moments are $m_g = \mathbb{E}_{\tau \sim
\pi_\theta}[s(\tau) \mid I(\tau)=1]$ and, using $I(\tau)^2 = I(\tau)$, $S_g =
\mathbb{E}_{\tau \sim \pi_\theta}[\|s(\tau)\|^2 \mid I(\tau)=1]$, and
\cref{ass:snr} with $X = s(\tau) I(\tau)$ supplies the ratio bound $R_g =
S_g/\|m_g\|^2 \le K$. Substituting, the relative variance is exactly
\begin{equation}
    \kappa^2[\nabla_\theta J_C(\theta)] = \frac{1}{pN} \left( \frac{S_g}{\|m_g\|^2} \right) - \frac{1}{N},
\end{equation}
bounded by
\begin{equation}
    \frac{1 - p}{pN} \le \kappa^2[\nabla_\theta J_C(\theta)] \le \frac{K}{pN},
\end{equation}
and hence
\begin{equation}
    \kappa^2[\nabla_\theta J_C(\theta)] = \Theta\left(\frac{1}{pN}\right),
\end{equation}
where the lower bound holds for every policy.

\subsection{CVaR: $N_\mathrm{eff}=\varepsilon N$}
\label[appendix]{app:cvar}

CVaR relies on the Rockafellar-Uryasev dual formulation, where $\eta$ converges
to the $(1-\varepsilon)$-quantile $\eta^*$, and the constraint of
OP \ref{eq:cvar-opt} evaluates the expected tail risk
\begin{equation}
    F(\theta) = \eta^* + \frac{1}{\varepsilon}\mathbb{E}_{\tau\sim \pi_\theta}\left[(C(\tau) - \eta^*)^+\right]
\end{equation}
against the cost limit. The CVaR constraint is therefore:
\begin{equation}
    J_C(\theta) = F(\theta) - \ell \le 0,
\end{equation}
with Monte Carlo estimator
\begin{equation}
    \hat{J}_C = \eta^* + \frac{1}{\varepsilon N} \sum_{i=1}^N (C(\tau_i) - \eta^*)^+ - \ell.
\end{equation}
This estimator holds $\eta$ fixed at the population quantile $\eta^*$; the
practical estimator instead minimizes over $\eta$, or equivalently plugs the
empirical $(1-\varepsilon)$-quantile $\hat{\eta}$ of the same batch into
\cref{eq:rockafellar}. Because $\eta^*$ minimizes the Rockafellar-Uryasev
objective (\ref{eq:rockafellar}), the objective is stationary in $\eta$ at
$\eta^*$ (differentiable there under \cref{ass:support}, which excludes an
atom at $\eta^*$), so the perturbation $\hat{\eta} - \eta^* =
\mathcal{O}_p(N^{-1/2})$ enters the estimator only at second order: the
empirical-quantile estimator is first-order equivalent to the fixed-quantile
estimator above, with the same asymptotic variance
\cite{shapiro2009lectures}. The scaling derived below therefore transfers to
the practical estimator to leading order in $1/N$.

In the form of \cref{eq:jc-decomp}, the sampled statistic is the scaled tail
penalty $x(\tau) = \frac{1}{\varepsilon}(C(\tau) - \eta^*)^+$ and the offset
is $b = \eta^* - \ell$. The active event is the tail event $\{C(\tau) >
\eta^*\}$. Throughout, we use $q = \mathbb{P}(C(\tau) > \eta^*) =
\varepsilon$, which holds under \cref{ass:support}: the cost return places
no probability mass at $\eta^*$, and its zero-cost atom $\mathbb{P}(C(\tau) = 0) = 1 - f$
lies below $\eta^* > 0$. Let $m_{\text{tail}} =
\mathbb{E}[(C(\tau) - \eta^*) \mid C(\tau) > \eta^*]$ and $S_{\text{tail}} =
\mathbb{E}[(C(\tau) - \eta^*)^2 \mid C(\tau) > \eta^*]$ denote the
conditional moments of the unscaled tail penalty, so that $m_a =
m_{\text{tail}}/\varepsilon$, $S_a = S_{\text{tail}}/\varepsilon^2$, and the
scaling cancels in the conditional ratio $R_x =
S_{\text{tail}}/m_{\text{tail}}^2$, which \cref{ass:snr} bounds. By
\cref{eq:jc-signal}, $F(\theta) = \eta^* + q m_a = \eta^* + m_{\text{tail}}$:
the tail mean $m_{\text{tail}}$ is the excess of CVaR over VaR, and the
signal is
\begin{equation}
\label{eq:cvar-signal}
    J_C(\theta) = \eta^* + m_{\text{tail}} - \ell,
\end{equation}
with signal share (\ref{eq:signal-share}) $r = m_{\text{tail}} /
|J_C(\theta)|$. Substituting into \cref{lem:value} with $q = \varepsilon$, the
relative variance is exactly
\begin{equation}
\label{eq:cvar-value-rv}
    \kappa^2[J_C(\theta)] = \frac{m_{\text{tail}}^2}{J_C(\theta)^2} \left[ \frac{1}{\varepsilon N}\left( \frac{S_{\text{tail}}}{m_{\text{tail}}^2} \right) - \frac{1}{N} \right],
\end{equation}
bounded by
\begin{equation}
    \frac{1 - \varepsilon}{\varepsilon N} \cdot \frac{m_{\text{tail}}^2}{J_C(\theta)^2} \le \kappa^2[J_C(\theta)] \le \frac{K}{\varepsilon N} \cdot \frac{m_{\text{tail}}^2}{J_C(\theta)^2},
\end{equation}
and hence, under \cref{ass:slack},
\begin{equation}
    \kappa^2[J_C(\theta)] = \Theta\left(\frac{1}{\varepsilon N}\right).
\end{equation}
For the CVaR constraint, \cref{ass:slack} reads: the slack between the cost
limit and the CVaR, $|\ell - F(\theta)|$, is proportional to the excess of
CVaR over VaR, $m_{\text{tail}}$. This is the CVaR analogue of the
indicator's $p = c\varepsilon$, and it is not vacuous. For a cost return with
bounded support, $\eta^*$ converges to the maximum cost return
$\operatorname{ess\,sup} C(\tau)$ and $m_{\text{tail}} \to 0$ as $\varepsilon
\to 0$ (for $C(\tau)$ uniform on $[0, 1]$, $m_{\text{tail}} = \varepsilon/2$
and $S_{\text{tail}} = \varepsilon^2/3$), so at a fixed slack $|J_C(\theta)|
\ge \zeta$ the exact expression (\ref{eq:cvar-value-rv}) gives
$\kappa^2[J_C(\theta)] \approx \varepsilon/(3 \zeta^2 N) \to 0$: the upper
bound $\mathcal{O}(1/(\varepsilon N))$ holds but is loose, and the constraint
value is estimated with vanishing relative error because the sampled part of
the constraint has itself vanished. The gradient, which has no offset, has no
such escape.

For the gradient, the constant cost limit vanishes under differentiation,
$\nabla_\theta J_C(\theta) = \nabla_\theta F(\theta)$, and the Envelope
Theorem allows $\eta^*$ to be held fixed while differentiating with respect
to the policy parameters $\theta$, so \cref{eq:grad-decomp} applies:
\begin{equation}
    \nabla_\theta J_C(\theta) = \frac{1}{\varepsilon} \mathbb{E}_{\tau \sim \pi_\theta} \left[ s(\tau) (C(\tau) - \eta^*)^+ \right]
    = \frac{1}{\varepsilon} \cdot \varepsilon \cdot \mathbb{E}_{\tau \sim \pi_\theta} \left[ s(\tau) (C(\tau) - \eta^*) \mid C(\tau) > \eta^* \right] = m,
\end{equation}
where $m = \mathbb{E}_{\tau \sim \pi_\theta}[s(\tau) (C(\tau) - \eta^*) \mid
C(\tau) > \eta^*]$ is the conditional gradient of the unscaled tail penalty:
the $1/\varepsilon$ of the estimator cancels the probability $\varepsilon$ of
the tail event, and the true gradient has squared magnitude $\|m\|^2$
independent of $\varepsilon$. The Monte Carlo estimator is $\hat{g} =
\frac{1}{\varepsilon N} \sum_{i=1}^N s(\tau_i) (C(\tau_i) - \eta^*)^+$;
because the $(x)^+$ operator evaluates to zero outside the tail event, only
an expected $\varepsilon N$ trajectories contribute to the sum, defining the
effective sample size $N_\mathrm{eff} = \varepsilon N$. With $S =
\mathbb{E}_{\tau \sim \pi_\theta}[\|s(\tau) (C(\tau) - \eta^*)\|^2 \mid
C(\tau) > \eta^*]$, the scaling cancels in $R_g = S/\|m\|^2$, which
\cref{ass:snr} with $X = s(\tau)(C(\tau) - \eta^*)^+$ bounds, so
substituting into \cref{lem:grad} with $q = \varepsilon$, the relative variance
is exactly
\begin{equation}
    \kappa^2[\nabla_\theta J_C(\theta)] = \frac{1}{\varepsilon N} \left( \frac{S}{\|m\|^2} \right) - \frac{1}{N},
\end{equation}
bounded by
\begin{equation}
    \frac{1 - \varepsilon}{\varepsilon N} \le \kappa^2[\nabla_\theta J_C(\theta)] \le \frac{K}{\varepsilon N},
\end{equation}
and hence
\begin{equation}
    \kappa^2[\nabla_\theta J_C(\theta)] = \Theta\left(\frac{1}{\varepsilon N}\right),
\end{equation}
where the lower bound holds for every policy, in particular for the bounded-support cost returns under which the
constraint value escapes the rate above.

\subsection{Cantelli: $N_\mathrm{eff}=fN$}
\label[appendix]{app:cantelli}

Canary enforces the constraint pair of OP \ref{eq:cantelli-objective}: the
Cantelli constraint $B(\theta) \le 0$ and the mean constraint $\mu(\theta) - \ell \le
0$. We derive the relative variance of each in turn, in both cases treating
the constraint value first and the policy gradient second.

\subsubsection{Cantelli Constraint}

The Canary formulation avoids algorithmic thresholding operators,
defining the Cantelli constraint $B(\theta)$ as in Equation \ref{eq:cpo-cheb},
which is already in canonical form, $J_C(\theta) \le 0$, with
\begin{equation}
\label{eq:cantelli-jc}
    J_C(\theta) = B(\theta) = \left (\frac{1}{\varepsilon} - 1\right ) \tilde{\mu}(\theta) + 2\ell \mu(\theta) - \frac{1}{\varepsilon}\mu(\theta)^2 - \ell^2.
\end{equation}
Write $\mu_i(\theta) = \mathbb{E}_{\tau \sim \pi_\theta}[C(\tau)^i]$ for the
$i$-th moment of the cost return, $i \in \{1, 2\}$, so that $\mu_1(\theta) =
\mu(\theta)$ and $\mu_2(\theta) = \tilde{\mu}(\theta)$. The Monte Carlo
estimators for these moments over a batch of $N$ trajectories are the
empirical means
\begin{equation}
\label{eq:moment-est}
    \hat{\mu}_i = \frac{1}{N} \sum_{j=1}^N C(\tau_j)^i.
\end{equation}
The constraint
value is a smooth function of the two moments, with sensitivities
\begin{equation}
\label{eq:cantelli-sens}
    c_1 = \frac{\partial B}{\partial \mu_1} = 2\left(\ell - \frac{\mu_1(\theta)}{\varepsilon}\right), \qquad
    c_2 = \frac{\partial B}{\partial \mu_2} = \frac{1}{\varepsilon} - 1,
\end{equation}
evaluated at the current $\theta$, in terms of which it can be written
exactly as
\begin{equation}
\label{eq:cantelli-lin}
    J_C(\theta) = c_1 \mu_1(\theta) + c_2 \mu_2(\theta) + b, \qquad
    b = J_C(\theta) - c_1 \mu_1(\theta) - c_2 \mu_2(\theta) = \frac{\mu(\theta)^2}{\varepsilon} - \ell^2.
\end{equation}
Replacing the moments by their empirical means gives the unbiased Monte Carlo
estimator of the constraint value,
\begin{equation}
\label{eq:cantelli-est}
    \hat{J}_C = c_1 \hat{\mu}_1 + c_2 \hat{\mu}_2 + b.
\end{equation}
$\hat{J}_C$ is unbiased but not computable, since $c_1$ and $b$ involve the
true $\mu(\theta)$; the computable, biased plug-in estimator that the
algorithm evaluates is analyzed once the rate for $\hat{J}_C$ is
established.

From \cref{eq:cantelli-lin}, the sampled statistic for $J_C(\theta)$ is given by
\begin{equation}
\label{eq:cantelli-z}
    x(\tau) = c_1 C(\tau) + c_2 C(\tau)^2.
\end{equation}
The offset $b$ of \cref{eq:cantelli-lin}, which is built from the true
moments at the current $\theta$ rather than from the sampled trajectories,
is an offset in the sense of \cref{eq:jc-decomp}: it depends on
$\theta$, but not on the batch.
The statistic $x(\tau)$ vanishes off the active event $\{C(\tau) > 0\}$,
which occurs with probability $q = f$, so the Cantelli constraint is in
canonical form and can be substituted into \cref{lem:value}.

Let $\mu_{k,a} =
\mathbb{E}_{\tau \sim \pi_\theta}[C(\tau)^k \mid C(\tau) > 0]$ for $k \in
\{1,2,3,4\}$ denote the cost moments conditional on the active trajectories,
so that $\mu_i(\theta) = f\mu_{i,a}$ (moments up to order four appear because
$x(\tau)^2$ involves $C(\tau)^4$). The conditional moments of the sampled
statistic are
\begin{equation}
    m_a = c_1 \mu_{1,a} + c_2 \mu_{2,a}, \qquad
    S_a = c_1^2 \mu_{2,a} + 2 c_1 c_2 \mu_{3,a} + c_2^2 \mu_{4,a},
\end{equation}
so the sampled part of the constraint value is $\mathbb{E}[x(\tau)] = f m_a
= c_1 \mu_1(\theta) + c_2 \mu_2(\theta)$ and the signal share
(\ref{eq:signal-share}) is
\begin{equation}
\label{eq:cantelli-share}
    r = \frac{|c_1 \mu_1(\theta) + c_2 \mu_2(\theta)|}{|J_C(\theta)|}.
\end{equation}
\Cref{lem:value} requires the conditional ratio $R_x = S_a/m_a^2$, which for
the linear combination $x(\tau)$ is not bounded by \cref{ass:snr} alone: the
numerator contains the cross moment $\mu_{3,a}$, and the denominator is a
squared linear combination that cancellation between its two terms can drive
toward zero. Bounding the numerator from above and the denominator from
below,
\begin{equation}
\label{eq:cantelli-value-nocancel}
\begin{split}
    R_x
    &= \frac{c_1^2 \mu_{2,a} + 2 c_1 c_2 \mu_{3,a} + c_2^2 \mu_{4,a}}{\left( c_1 \mu_{1,a} + c_2 \mu_{2,a} \right)^2} \\
    &\le \frac{\left( |c_1| \sqrt{\mu_{2,a}} + |c_2| \sqrt{\mu_{4,a}} \right)^2}{\left( c_1 \mu_{1,a} + c_2 \mu_{2,a} \right)^2}
    \qquad \text{(Cauchy--Schwarz: } \mu_{3,a} \le \sqrt{\mu_{2,a}\,\mu_{4,a}}\text{)} \\
    &\le \frac{2 \left( c_1^2 \mu_{2,a} + c_2^2 \mu_{4,a} \right)}{\left( c_1 \mu_{1,a} + c_2 \mu_{2,a} \right)^2}
    \qquad \text{(} (a + b)^2 \le 2(a^2 + b^2) \text{)} \\
    &\le \frac{2K \left( c_1^2 \mu_{1,a}^2 + c_2^2 \mu_{2,a}^2 \right)}{\left( c_1 \mu_{1,a} + c_2 \mu_{2,a} \right)^2}
    \qquad \text{(\cref{ass:snr}: } \mu_{2,a} \le K \mu_{1,a}^2,\ \mu_{4,a} \le K \mu_{2,a}^2\text{)} \\
    &\le \frac{2K \left( c_1^2 \mu_{1,a}^2 + c_2^2 \mu_{2,a}^2 \right)}{\nu \left( c_1^2 \mu_{1,a}^2 + c_2^2 \mu_{2,a}^2 \right)}
    \qquad \text{(\cref{ass:no-cancel}, moment pair)} \\
    &= \frac{2K}{\nu} =: K_x.
\end{split}
\end{equation}
In the last inequality, \cref{ass:no-cancel} is stated for the unconditioned
moments $\mu_i(\theta) = f\mu_{i,a}$, and the factor $f^2$ cancels on both
sides. For this scalar moment pair, cancellation requires $c_1 < 0$, since
$c_2 = \frac{1}{\varepsilon} - 1 > 0$ and both conditional moments are
positive; that is, it requires $\mu(\theta) > \varepsilon \ell$, and then
additionally $|c_1| \mu_{1,a} \approx c_2 \mu_{2,a}$. Substituting into
\cref{lem:value} with $q = f$, signal share $r$ of
\cref{eq:cantelli-share}, and $R_x \le K_x$, the relative variance of the
constraint evaluation is bounded by
\begin{equation}
\label{eq:cantelli-rv}
    \frac{\omega\, r^2}{fN} \le \kappa^2[J_C(\theta)] \le \frac{r^2}{fN} \cdot \frac{2K}{\nu},
\end{equation}
where the lower bound follows from \cref{eq:value-exact} and the
non-degeneracy clause $R_x \ge 1 + \omega$ of \cref{ass:slack}, since $R_x - f
\ge \omega$ for every $f \in (0, 1]$, and hence, with
\cref{ass:slack} bounding $r$ above and below,
\begin{equation}
    \kappa^2[J_C(\theta)] = \Theta\left(\frac{1}{fN}\right).
\end{equation}
The bound above is uniform in $f$,
but \cref{ass:slack} requires the signal share $r$ of
\cref{eq:cantelli-share} to be uniform in $\varepsilon$ as well. Expanding the
sensitivities, $c_1 \mu_1(\theta) + c_2 \mu_2(\theta) = J_C(\theta) + \ell^2 -
\mu(\theta)^2/\varepsilon$, so
\begin{equation}
    r = \left| 1 + \frac{\ell^2 - \mu(\theta)^2/\varepsilon}{J_C(\theta)} \right|.
\end{equation}
By the identity $J_C(\theta) = \left(\frac{1}{\varepsilon} - 1\right)
\sigma^2(\theta) - (\mu(\theta) - \ell)^2$, for policies whose cost variance
is bounded below, $\sigma^2(\theta) \ge \sigma^2_{\min} > 0$, both $J_C(\theta)$
and the correction $\ell^2 - \mu(\theta)^2/\varepsilon$ grow at the rate
$1/\varepsilon$, and $r \to |\sigma^2(\theta) - \mu(\theta)^2| /
\sigma^2(\theta)$ as $\varepsilon \to 0$. This limit is bounded above by $1
+ \mu(\theta)^2/\sigma^2_{\min}$, and bounded away from zero unless the cost
variance equals the squared mean cost, $\tilde{\mu}(\theta) = 2\mu(\theta)^2$,
a knife-edge identity between the two moments. The gradient rate derived
later needs no such condition.

\paragraph{Constraint value plug-in estimator.}
The algorithm does not evaluate $\hat{J}_C$ itself, but the biased plug-in
estimator in which $\hat{\mu}_1$ also replaces $\mu(\theta)$ inside $c_1$
and $b$:
\begin{equation}
\label{eq:cantelli-plugin}
    \hat{B} = \beta \hat{\mu}_2 + 2\ell \hat{\mu}_1 - \frac{\hat{\mu}_1^2}{\varepsilon} - \ell^2.
\end{equation}
$\hat{J}_C$ is the first-order surrogate (the Delta Method) of $\hat{B}$,
and the two differ by the exact remainder
\begin{equation}
\label{eq:cantelli-remainder}
    \hat{B} - \hat{J}_C = -\frac{(\hat{\mu}_1 - \mu_1(\theta))^2}{\varepsilon},
\end{equation}
whose mean is $-\operatorname{Var}(\hat{\mu}_1)/\varepsilon =
-\sigma^2(\theta)/(\varepsilon N)$: the absolute order is
$\mathcal{O}(1/(\varepsilon N))$, not $\mathcal{O}(1/N)$, and the factor
$1/\varepsilon$ cannot be discarded in the regime $\varepsilon \to 0$ this
appendix concerns. The remainder is nevertheless second order relative to
the signal, which carries the same factor: \cref{ass:slack} gives
$|J_C(\theta)| \ge |c_1 \mu_1(\theta) + c_2 \mu_2(\theta)| / r_{\max}$,
\cref{ass:no-cancel} gives $|c_1 \mu_1(\theta) + c_2 \mu_2(\theta)| \ge
\sqrt{\nu}\, c_2 \mu_2(\theta)$, and $\sigma^2(\theta) \le \mu_2(\theta)$,
so with $c_2 = (1-\varepsilon)/\varepsilon$,
\begin{equation}
    \frac{\left|\mathbb{E}[\hat{B} - \hat{J}_C]\right|}{|J_C(\theta)|}
    \le \frac{\sigma^2(\theta)}{\varepsilon N} \cdot
    \frac{r_{\max}\, \varepsilon}{(1-\varepsilon)\sqrt{\nu}\, \mu_2(\theta)}
    \le \frac{r_{\max}}{(1-\varepsilon)\sqrt{\nu}\, N},
\end{equation}
uniformly in $(p, \varepsilon, f)$ for $\varepsilon$ bounded away from one.
For the variance of the remainder,
$\operatorname{Var}((\hat{\mu}_1 - \mu_1(\theta))^2)/\varepsilon^2 \le
\mathbb{E}[(\hat{\mu}_1 - \mu_1(\theta))^4]/\varepsilon^2$, and for an
empirical mean $\mathbb{E}[(\hat{\mu}_1 - \mu_1(\theta))^4] \le
3\sigma^4(\theta)/N^2 + \mathbb{E}[(C(\tau) - \mu_1(\theta))^4]/N^3$, where
the fourth moment is bounded through \cref{ass:snr} ($\mu_{4,a} \le K
\mu_{2,a}^2$, so $\mathbb{E}[C(\tau)^4] \le K \mu_2(\theta)^2 / f$).
Normalizing by $J_C(\theta)^2 \ge \nu c_2^2 \mu_2(\theta)^2 / r_{\max}^2$ as
above and using $\sigma^4(\theta) \le \mu_2(\theta)^2$, the remainder
contributes to the relative variance at order $1/N^2 + 1/(fN^3)$, smaller
than the $\Theta(1/(fN))$ rate above by at least a factor of $1/N$,
uniformly in $(p, \varepsilon, f)$. The cross term between the sampling
error of $\hat{J}_C$ and the remainder is bounded by Cauchy--Schwarz at the
geometric mean of the two contributions, of relative order
$1/(\sqrt{f}N^{3/2})$, and the squared bias is of relative order $1/N^2$ by
the mean bound above; both are likewise subleading. The constraint-value
rate (\ref{eq:cantelli-rv}) therefore holds for $\hat{B}$ to leading order.

The Cantelli constraint policy gradient is
\begin{equation}
\label{eq:cantelli-grad}
    \nabla_\theta J_C(\theta) = c_1 \nabla_\theta \mu_1(\theta) + c_2 \nabla_\theta \mu_2(\theta)
    = \mathbb{E}_{\tau \sim \pi_\theta}\left[ s(\tau) \left( c_1 C(\tau) + c_2 C(\tau)^2 \right) \right]
    = \mathbb{E}_{\tau \sim \pi_\theta}[s(\tau) x(\tau)],
\end{equation}
where
$\nabla_\theta \mu_i(\theta) = \mathbb{E}_{\tau \sim
\pi_\theta}[s(\tau)C(\tau)^i] =: m_i(\theta)$, and $x(\tau)$ is the sampled
statistic of \cref{eq:cantelli-z}. The offset $b$ of
\cref{eq:cantelli-lin} depends on $\theta$, so the sufficient condition
of \cref{lem:grad} fails, but the gradient identity \cref{eq:grad-decomp} it
is used for holds regardless: differentiating the decomposition $J_C(\theta)
= c_1 \mu_1(\theta) + c_2 \mu_2(\theta) + b$ produces, beyond $c_1
\nabla_\theta \mu_1(\theta) + c_2 \nabla_\theta \mu_2(\theta)$, the two terms
$\mu_1(\theta) \nabla_\theta c_1 = -\frac{2\mu_1(\theta)}{\varepsilon}
\nabla_\theta \mu_1(\theta)$ and $\nabla_\theta b =
+\frac{2\mu_1(\theta)}{\varepsilon} \nabla_\theta \mu_1(\theta)$, which
cancel exactly, recovering \cref{eq:cantelli-grad}. \Cref{lem:grad} therefore
applies exactly.

Let $m_{i,a} = \mathbb{E}_{\tau \sim \pi_\theta}[s(\tau)C(\tau)^i
\mid C(\tau) > 0]$ and $S_{i,a} = \mathbb{E}_{\tau \sim
\pi_\theta}[\|s(\tau)C(\tau)^i\|^2 \mid C(\tau) > 0]$ denote the moment
gradients and their second moments conditional on the active trajectories,
so that $m_i(\theta) = fm_{i,a}$, with empirical moment gradients $\hat{m}_i =
\frac{1}{N}\sum_{j=1}^N s(\tau_j)C(\tau_j)^i$ averaging over all $N$
trajectories including the exact zeros.

The Monte Carlo estimator of the constraint gradient is the corresponding
linear combination of the empirical moment gradients,
$\hat{g} = c_1 \hat{m}_1 + c_2 \hat{m}_2$, with the sensitivity $c_1$
evaluated at the true moment $\mu_1(\theta)$. Like $\hat{J}_C$
(\ref{eq:cantelli-est}), this fixed-coefficient estimator is not computable;
the practical estimator plugs $\hat{\mu}_1$ into $c_1$. We analyze the
fixed-coefficient estimator first and bound the plug-in perturbation, which
unlike the value remainder (\ref{eq:cantelli-remainder}) is not higher
order, at the end of this derivation. By linearity, $\hat{g} = \frac{1}{N} \sum_{j=1}^N
s(\tau_j) x(\tau_j)$ exactly, the gradient estimator of \cref{lem:grad} for
the sampled statistic $x(\tau)$ of \cref{eq:cantelli-z}, with conditional
gradient moments $m_g = c_1 m_{1,a} + c_2 m_{2,a}$ and $S_g =
c_1^2 S_{1,a} + 2 c_1 c_2 S_{12,a} + c_2^2 S_{2,a}$, where the conditional cross moment is $S_{12,a} =
\mathbb{E}_{\tau \sim \pi_\theta}[\langle s(\tau)C(\tau), s(\tau)C(\tau)^2
\rangle \mid C(\tau) > 0] = \mathbb{E}_{\tau \sim \pi_\theta}[\|s(\tau)\|^2
C(\tau)^3 \mid C(\tau) > 0]$. \Cref{lem:grad} requires the conditional ratio
$R_g = S_g/\|m_g\|^2$, bounded exactly as $R_x$ in
\cref{eq:cantelli-value-nocancel}, with \cref{ass:no-cancel} now applied to
the moment-gradient pair:
\begin{equation}
\label{eq:cantelli-grad-nocancel}
\begin{split}
    R_g
    &= \frac{c_1^2 S_{1,a} + 2 c_1 c_2 S_{12,a} + c_2^2 S_{2,a}}{\left\| c_1 m_{1,a} + c_2 m_{2,a} \right\|^2} \\
    &\le \frac{\left( |c_1| \sqrt{S_{1,a}} + |c_2| \sqrt{S_{2,a}} \right)^2}{\left\| c_1 m_{1,a} + c_2 m_{2,a} \right\|^2}
    \qquad \text{(Cauchy--Schwarz: } S_{12,a} \le \sqrt{S_{1,a}\,S_{2,a}}\text{)} \\
    &\le \frac{2 \left( c_1^2 S_{1,a} + c_2^2 S_{2,a} \right)}{\left\| c_1 m_{1,a} + c_2 m_{2,a} \right\|^2}
    \qquad \text{(} (a + b)^2 \le 2(a^2 + b^2) \text{)} \\
    &\le \frac{2K \left( c_1^2 \|m_{1,a}\|^2 + c_2^2 \|m_{2,a}\|^2 \right)}{\left\| c_1 m_{1,a} + c_2 m_{2,a} \right\|^2}
    \qquad \text{(\cref{ass:snr}: } S_{i,a} \le K \|m_{i,a}\|^2\text{)} \\
    &\le \frac{2K \left( c_1^2 \|m_{1,a}\|^2 + c_2^2 \|m_{2,a}\|^2 \right)}{\nu \left( c_1^2 \|m_{1,a}\|^2 + c_2^2 \|m_{2,a}\|^2 \right)}
    \qquad \text{(\cref{ass:no-cancel}, moment-gradient pair)} \\
    &= \frac{2K}{\nu} =: K_g.
\end{split}
\end{equation}
As before, \cref{ass:no-cancel} is stated for the unconditioned gradients
$m_i(\theta) = f m_{i,a}$ and the factor $f^2$ cancels on both sides.
Substituting into \cref{lem:grad} with $q = f$ and $R_g \le K_g$, the relative
variance of the gradient is exactly bounded by
\begin{equation}
    \frac{\omega}{fN} \le \kappa^2[\nabla_\theta J_C(\theta)] \le \frac{1}{fN} \cdot \frac{2K}{\nu},
\end{equation}
where the lower bound follows from the exact form in \cref{lem:grad} and
$R_g \ge 1 + \omega$ (\cref{ass:slack}), and hence
\begin{equation}
    \kappa^2[\nabla_\theta J_C(\theta)] = \Theta\left(\frac{1}{fN}\right).
\end{equation}

\paragraph{Policy gradient plug-in estimator.}
The practical gradient estimator plugs $\hat{\mu}_1$ into the sensitivity,
\begin{equation}
    \hat{g}_B = \hat{c}_1 \hat{m}_1 + c_2 \hat{m}_2, \qquad
    \hat{c}_1 = 2\left(\ell - \frac{\hat{\mu}_1}{\varepsilon}\right),
\end{equation}
and differs from the fixed-coefficient estimator $\hat{g}$ analyzed above by
the exact perturbation
\begin{equation}
\label{eq:cantelli-grad-perturb}
    \hat{g}_B - \hat{g} = -\frac{2}{\varepsilon}\left(\hat{\mu}_1 - \mu_1(\theta)\right) \hat{m}_1.
\end{equation}
Unlike the value remainder (\ref{eq:cantelli-remainder}), which is quadratic
in $\hat{\mu}_1 - \mu_1(\theta)$, this perturbation is linear in it, of
order $\mathcal{O}_p(1/(\varepsilon\sqrt{N}))$ --- the same order as the
Monte Carlo error of $\hat{g}$ itself, whose coefficients also carry the
factor $1/\varepsilon$. It is therefore a first-order term and must be
bounded against the signal. Decomposing $\hat{m}_1 = m_1(\theta) +
(\hat{m}_1 - m_1(\theta))$,
\begin{equation}
    \mathbb{E}\left[\left\|\hat{g}_B - \hat{g}\right\|^2\right]
    \le \frac{8}{\varepsilon^2} \operatorname{Var}(\hat{\mu}_1)\, \|m_1(\theta)\|^2
    + \frac{8}{\varepsilon^2}\, \mathbb{E}\left[\left(\hat{\mu}_1 - \mu_1(\theta)\right)^2 \left\|\hat{m}_1 - m_1(\theta)\right\|^2\right],
\end{equation}
where the second term is of relative order $1/N^2$ under fourth-moment
analogues of \cref{ass:snr} for $C(\tau)$ and $s(\tau)C(\tau)$. For the
first term, $\operatorname{Var}(\hat{\mu}_1) \le f\mu_{2,a}/N$ and
$\|m_1(\theta)\|^2 = f^2\|m_{1,a}\|^2$, while \cref{ass:no-cancel} bounds
the signal from below through the $c_2$ component alone,
$\|\nabla_\theta J_C(\theta)\|^2 \ge \nu c_2^2 \|m_2(\theta)\|^2 = \nu c_2^2
f^2 \|m_{2,a}\|^2$ with $c_2 = (1-\varepsilon)/\varepsilon$; the $c_1$
component cannot serve here, since $c_1$ passes through zero at
$\mu(\theta) = \varepsilon\ell$ while the perturbation does not vanish.
Provided the conditional cross-moment condition
\begin{equation}
\label{eq:cross-moment}
    \mu_{2,a} \|m_{1,a}\|^2 \le K_2 \|m_{2,a}\|^2
\end{equation}
holds for a constant $K_2$ independent of $(p, \varepsilon, f)$ --- the
second-moment gradient carries at least the conditional cost scale of the
first, the analogue for gradients of $\mu_{2,a} \ge \mu_{1,a}^2$ --- the
relative contribution of the first term is
\begin{equation}
    \frac{8 f^3 \mu_{2,a} \|m_{1,a}\|^2 / (\varepsilon^2 N)}
         {\nu (1-\varepsilon)^2 f^2 \|m_{2,a}\|^2 / \varepsilon^2}
    \le \frac{8 K_2}{\nu (1-\varepsilon)^2} \cdot \frac{f}{N},
\end{equation}
of order $f/N \le 1/(fN)$: the perturbation does not raise the
$\mathcal{O}(1/(fN))$ upper bound. The plug-in estimator is biased,
$\mathbb{E}[\hat{g}_B - \hat{g}] =
-\frac{2}{\varepsilon}\operatorname{Cov}(\hat{\mu}_1, \hat{m}_1)$, but only
at relative order $1/N$ by Cauchy--Schwarz together with \cref{ass:snr} and
\cref{eq:cross-moment}, so the squared bias is subleading and relative
variance remains the fidelity measure. Since
$\operatorname{Var}(\hat{g}_B) \le
2\operatorname{Var}(\hat{g}) + 2\,\mathbb{E}[\|\hat{g}_B -
\hat{g}\|^2]$,
\begin{equation}
    \frac{\operatorname{Var}(\hat{g}_B)}{\left\|\nabla_\theta J_C(\theta)\right\|^2}
    \le 2\,\kappa^2[\nabla_\theta J_C(\theta)]
    + \frac{2\,\mathbb{E}\left[\left\|\hat{g}_B - \hat{g}\right\|^2\right]}{\left\|\nabla_\theta J_C(\theta)\right\|^2}
    = \mathcal{O}\left(\frac{1}{fN}\right):
\end{equation}
the relative variance of the practical gradient estimator obeys the same
$\mathcal{O}(1/(fN))$ upper bound. The perturbation
(\ref{eq:cantelli-grad-perturb}) is correlated with the sampling error of
$\hat{g}$, so it can in principle also lower the variance; the two-sided
$\Theta(1/(fN))$ statement of \cref{prop:ess} is for the unbiased
fixed-coefficient estimators $\hat{J}_C$ and $\hat{g}$, and the plug-in
estimators $\hat{B}$ and $\hat{g}_B$ satisfy the
$\mathcal{O}(1/(fN))$ upper bound in relative variance.

When the environmental cost signal is dense (e.g., energy consumption,
continuous proximity penalties), essentially every trajectory incurs some cost,
yielding $f \approx 1$ and $N_\mathrm{eff} \approx N$. Here, the relative
variances of both the constraint value and its gradient remain
bounded by $\mathcal{O}(1/N)$ with constants independent of the violation
threshold $\varepsilon$. This resolves the divergence of the baseline methods in
feasible, dense-cost regimes.

\subsubsection{Mean Constraint}

The mean constraint is already in canonical form,
\begin{equation}
    J_C(\theta) = \mu(\theta) - \ell \le 0,
\end{equation}
with Monte Carlo estimator $\hat{J}_C = \hat{\mu}_1 - \ell$ and $\hat{\mu}_1$
as in \cref{eq:moment-est}.

This is of the form \cref{eq:jc-decomp} with sampled statistic $x(\tau) =
C(\tau)$, $q = f$, and offset $b = -\ell$, so \cref{lem:value} applies. The sampled part is
$\mu(\theta) = f\mu_{1,a}$ (\ref{eq:jc-signal}) and the signal is
$J_C(\theta) = \mu(\theta) - \ell$, so the signal share
(\ref{eq:signal-share}) is $r = \mu(\theta) / |\mu(\theta) - \ell|$, and the
conditional ratio is $R_x = \mu_{2,a}/\mu_{1,a}^2$, bounded by
\cref{ass:snr}. Substituting into \cref{lem:value} with $q = f$, the relative
variance is exactly
\begin{equation}
    \kappa^2[J_C(\theta)] = r^2 \left[ \frac{1}{fN} \left( \frac{\mu_{2,a}}{\mu_{1,a}^2} \right) - \frac{1}{N} \right],
\end{equation}
bounded by
\begin{equation}
    \frac{\omega\, r^2}{fN} \le \kappa^2[J_C(\theta)] \le \frac{r^2 K}{fN},
\end{equation}
using $R_x \ge 1 + \omega$ (\cref{ass:slack}) for the lower bound, and hence, under \cref{ass:slack},
\begin{equation}
    \kappa^2[J_C(\theta)] = \Theta\left(\frac{1}{fN}\right).
\end{equation}
The structure mirrors the Cantelli constraint value (\ref{eq:cantelli-rv})
with the single moment in place of the linear combination $x(\tau)$, so no
analogue of \cref{ass:no-cancel} is needed. For the mean constraint,
\cref{ass:slack} reads: the policy uses a fixed fraction of the cost limit,
$\mu(\theta) = c\ell$ with $c \in (0, 1)$ fixed, giving $r = c/(1-c)$
exactly as for the indicator constraint.

For the policy gradient, $\nabla_\theta J_C(\theta) = \nabla_\theta
\mu(\theta) = m_1(\theta) = \mathbb{E}_{\tau \sim \pi_\theta}[s(\tau)C(\tau)]$,
with no sensitivities and no composition. The offset $b = -\ell$ is constant
in $\theta$, so \cref{lem:grad} applies with sampled statistic $x(\tau) =
C(\tau)$, $q = f$, conditional gradient moments $m_g = m_{1,a}$ and $S_g =
S_{1,a}$, and estimator $\hat{m}_1$: the relative variance is exactly
\begin{equation}
    \kappa^2[\nabla_\theta J_C(\theta)] = \frac{1}{fN} \left( \frac{S_{1,a}}{\|m_{1,a}\|^2} \right) - \frac{1}{N},
\end{equation}
bounded by
\begin{equation}
    \frac{\omega}{fN} \le \kappa^2[\nabla_\theta J_C(\theta)] \le \frac{K}{fN},
\end{equation}
and hence
\begin{equation}
    \kappa^2[\nabla_\theta J_C(\theta)] = \Theta\left(\frac{1}{fN}\right),
\end{equation}
which is \cref{lem:grad} with $R_g = S_{1,a}/\|m_{1,a}\|^2$ bounded above by
\cref{ass:snr} and below by \cref{ass:slack}. 

\section{Canary Algorithm}
\label[appendix]{app:algorithm}

\begin{algorithm}[H]
   \caption{Canary}
   \label{alg:var_cpo}
\begin{algorithmic}
   \STATE {\bfseries Input:} Initial policy $\pi_{\theta_0}$, value function estimates $V_{\phi}, V^{C}_{\psi}, V^{\tilde{C}}_{\chi}$, cost limit $\ell$, violation threshold $\varepsilon$, KL-divergence limit $\delta$.
   \STATE {\bfseries Initialize:} $\beta \leftarrow \frac{1}{\varepsilon} - 1$
   \FOR{$k = 0, 1, 2, ...$}
      \STATE \textbf{1. Data Collection:}
      \STATE Sample trajectories $\mathcal{D} = \{\tau_i\}$ using policy $\pi_{\theta_k}$.
      \STATE Compute augmented state $x(s_t, y_t, \gamma_c^t)$ (\ref{eq:aug-state}).

      \STATE \textbf{2. Advantage \& Return Estimation:}
      \STATE Value function estimates $V_{\phi}, V^{C}_{\psi}, V^{\tilde{C}}_{\chi}$ are used to estimate advantages $A_{\theta_k}, A_{\theta_k}^{C}, A_{\theta_k}^{\tilde{C}}$ respectively, employing the Generalized Advantage Estimation (GAE) \cite{gae}.
      \STATE Estimate expected cost returns $\mu(\theta_k)$ and augmented cost returns $\tilde{\mu}(\theta_k)$ using a Monte Carlo (MC) form with $\mathcal{D}$.

      \STATE \textbf{3. Constraint Construction:}
      \STATE Calculate policy-dependent bound $\rho(\theta_k)$.
      \STATE Set Cantelli constraint $(c_B, b_B)$ via (\ref{eq:taylor-c}), (\ref{eq:taylor-b}).
      \STATE Set mean constraint $(c_\mu, b_\mu)$ via (\ref{eq:taylor-cmu}).

      \STATE \textbf{4. Policy Update:}
      \STATE Compute objective gradient $g \leftarrow \nabla_\theta L(\theta)|_{\theta_k}$.
      \STATE Compute reachable minima $m_\mu$, $m_B$ (\ref{eq:reachability}) in closed form (\cref{app:solver}).
      \STATE Select the tier of (\ref{eq:recovery}) by reachability and solve it by active-set enumeration (\cref{app:solver}).

      \STATE \textbf{5. Critic Update:}
      \STATE Update $V_{\phi}$, $V^{C}_{\psi}$ and $V^{\tilde{C}}_{\chi}$ by minimizing MSE against return targets.
   \ENDFOR
\end{algorithmic}
\end{algorithm}

The value function estimates here are parameterized models; however, the Canary algorithm is agnostic to the estimate used.

\subsection{Closed-Form Solver}
\label[appendix]{app:solver}

In step coordinates $s = \theta - \theta_k$, the Tier a update
(\ref{eq:practical-update}) is
\begin{equation*}
    \max_s \; g^\top s \quad \text{s.t.} \;\; c_B + b_B^\top s \le 0, \;\;
    c_\mu + b_\mu^\top s \le 0, \;\; \tfrac{1}{2} s^\top H s \le \delta.
\end{equation*}
Its solution lies in the span of $H^{-1}g$, $H^{-1}b_B$, and $H^{-1}b_\mu$, so
after three conjugate-gradient solves the problem reduces to scalar algebra
on
\begin{equation*}
    q = g^\top H^{-1} g, \qquad
    r = \begin{pmatrix} b_B^\top H^{-1} g \\ b_\mu^\top H^{-1} g \end{pmatrix},
    \qquad
    S = \begin{pmatrix}
        b_B^\top H^{-1} b_B & b_B^\top H^{-1} b_\mu \\
        b_\mu^\top H^{-1} b_B & b_\mu^\top H^{-1} b_\mu
    \end{pmatrix}.
\end{equation*}

\paragraph{Reachability.} The minimum of a linear function $c' + b'^\top s$
over the trust region alone is $c' - \sqrt{2\delta\, b'^\top H^{-1} b'}$,
which gives $m_\mu$ directly from $(c_\mu, b_\mu)$. $m_B$ is the minimum of
$c_B + b_B^\top s$ over the trust region intersected with the half-space $c_\mu +
b_\mu^\top s \le 0$; this is a linear objective under one linear constraint
and the trust region, solved in closed form by the same one-active-constraint
case analysis as below.

\paragraph{Active-set enumeration.} With a linear objective the trust region
is active at any optimum. Writing $\nu = (\nu_1, \nu_2) \ge 0$ for the
multipliers of the linearized Cantelli and mean constraints and $\lambda > 0$ for the trust
region, stationarity gives $s^* = \frac{1}{\lambda} H^{-1}(g - \nu_1 b_B -
\nu_2 b_\mu)$, and the exact solution is found by enumerating which of the two
bounds is active:
\begin{itemize}[nosep]
    \item \emph{Neither}: $s^* = \sqrt{2\delta / q}\, H^{-1} g$, the natural
    gradient step.
    \item \emph{One}: the analytic CPO dual with the corresponding pair
    $(c', b')$ \cite{cpo}.
    \item \emph{Both}: tightness of the two bounds is linear in
    the multipliers, $\nu(\lambda) = S^{-1}(r + \lambda \mathbf{c})$ with
    $\mathbf{c} = (c_B, c_\mu)^\top$, and substituting into $\tfrac{1}{2}
    s^{*\top} H s^* = \delta$ the cross terms cancel, leaving
    \begin{equation*}
        \lambda^* = \sqrt{\frac{q - r^\top S^{-1} r}
        {2\delta - \mathbf{c}^\top S^{-1} \mathbf{c}}},
        \qquad \nu^* = S^{-1}(r + \lambda^* \mathbf{c}).
    \end{equation*}
    Here $q - r^\top S^{-1} r \ge 0$ is the squared $H^{-1}$-norm of the
    component of $g$ orthogonal to the two constraint gradients, and $2\delta
    > \mathbf{c}^\top S^{-1} \mathbf{c}$ is precisely the condition that both
    tightened constraints intersect inside the trust region (otherwise the
    case is skipped).
\end{itemize}
A candidate is retained only if it is primal feasible with non-negative
multipliers; among retained candidates the one with the largest objective
value is the exact solution. If $S$ is near-singular ($b_B$ and $b_\mu$ close
to parallel), the mean constraint is directionally redundant and is dropped.
The recovery tiers of (\ref{eq:recovery}) belong to the same family: Tier b
minimizes the linear objective $b_B^\top s$ under the mean constraint and the
trust region (the one-active-constraint case), and Tier c is the
unconstrained trust-region minimizer $s^* = -\sqrt{2\delta / S_{22}}\,
H^{-1} b_\mu$.

\section{Experimental Details}
\label[appendix]{app:hyperparams}

\subsection{Canary Hyperparameters}

\begin{table}[H]
\centering
\label{tab:hyperparams}
\begin{tabular}{llc}
\toprule
\textbf{Hyperparameter} & \textbf{Value} \\
Hidden Layers & 2 \\
Hidden Units  & 256 \\
Activation    & tanh \\
Learning Rate & $3 \times 10^{-4}$ \\
Optimizer     & Adam \\
Grad Norm Clip     & 0.5 \\
GAE $\lambda$ & 0.95 \\
Reward Discount $\gamma$ & 0.99 \\
Cost Discount $\gamma_c$ & 1.0 \\
Trust Region $\delta$ & 0.01 \\
Critic epochs & 80 \\
Base Seed & 0 \\
\bottomrule
\end{tabular}
\caption{Canary Hyperparameter Settings}
\end{table}

To assess performance, trajectories are drawn from distributions that depend
on random number generator (RNG) seeds, so for statistical trustworthiness we
measure results over policies trained on multiple seeds. The base seed
initializes a single PRNG key, from which the parallel multi-seed runs
reported in \cref{sec:results} each receive an independent key by splitting.

\subsection{Baseline Hyperparameters and Tuning Budget}
\label[appendix]{app:baselines}

All methods share the same actor-critic trunk (two hidden layers of 256 units,
$\tanh$), optimizer (Adam at $3\times10^{-4}$, linearly annealed to zero),
gradient-norm clip ($0.5$), reward discount ($\gamma = 0.99$), cost discount
($\gamma_c = 1.0$), and GAE parameter ($\lambda = 0.95$), and are run for 30
seeds. CVaR-CPO additionally carries a distributional cost critic (two hidden
layers of 64 units, $\tanh$) that the other methods do not require.
Environment-level settings, shared across methods, are given in
\cref{tab:env-settings}.

\begin{table}[H]
\centering
\begin{tabular}{lccc}
\toprule
\textbf{Setting} & \textbf{EcoAnt} & \textbf{PointGoal} & \textbf{IcyLake} \\
\midrule
Cost limit $\ell$        & $150$ & $100$ & $5.5$ \\
Violation threshold $\varepsilon$ & $0.02$ & $0.02$ & $0.02$ \\
Parallel environments    & $20$  & $20$  & $5$ \\
Rollout length           & $100$ & $125$ & $200$ \\
Total environment steps  & $10$M & $2$M  & $1$M \\
\bottomrule
\end{tabular}
\caption{Environment-level settings, identical across all methods.}
\label{tab:env-settings}
\end{table}

The trust-region methods (Canary, CPO, CVaR-CPO) share CPO's update machinery
and its settings: trust region $\delta = 0.01$, backtracking coefficient $0.8$
with at most $10$ backtracking iterations, conjugate-gradient damping $0.1$, and
no entropy bonus. They differ only in how the constraint pair $(c, b)$ is
constructed. CPO is applied to the breach-indicator variant of each environment
with cost limit $\varepsilon$ and an undiscounted indicator return, so that it
constrains $\mathbb{P}(C(\tau) > \ell) \le \varepsilon$ directly
(\cref{sec:indicator}). Per-method settings for the first-order methods are
given in \cref{tab:baseline-hyperparams}.

\begin{table}[H]
\centering
\begin{tabular}{llccc}
\toprule
\textbf{Method} & \textbf{Hyperparameter} & \textbf{EcoAnt} & \textbf{PointGoal} & \textbf{IcyLake} \\
\midrule
\multirow{4}{*}{PPO}
  & Ratio clip        & $0.2$    & $0.2$    & $0.2$ \\
  & Epochs per update & $4$      & $4$      & $4$ \\
  & Minibatch size    & $500$    & $500$    & $40$ \\
  & Entropy coeff.    & $10^{-5}$ & $0.0075$ & $0.01$ \\
\midrule
\multirow{6}{*}{CPPO}
  & $\lambda$ init      & $50$     & $50$     & $1$ \\
  & $\lambda$ step size & $0.03$   & $0.01$   & $0.03$ \\
  & PID gain $k_p$      & $600$    & $60$     & $15$ \\
  & CVaR clip ratio     & $1000$   & $300$    & $300$ \\
  & Epochs per update   & $10$     & $4$      & $4$ \\
  & Entropy coeff.      & $0.0075$ & $0.0075$ & $0.01$ \\
\bottomrule
\end{tabular}
\caption{Per-method hyperparameters for the first-order baselines. Ratio clip
and value-loss coefficients follow standard PPO practice.}
\label{tab:baseline-hyperparams}
\end{table}

\paragraph{Tuning budget.} CPPO is the only method whose hyperparameters were
selected by search, because its Lagrangian gains have no established default in
this setting. For each environment we swept a $12$-point grid: on EcoAnt and
PointGoal, $k_p \in \{60, 100, 150\}$ crossed with $\lambda$ step size $\in
\{0.01, 0.03\}$ and $\lambda$ initialization $\in \{1, 50\}$ over $2$ seeds per
configuration, extended along the gain axis (up to $k_p = 2400$ with $\lambda$
cap $4000$ on EcoAnt, $k_p = 600$ on PointGoal); on IcyLake, $k_p \in \{0, 15,
60, 150\}$ crossed with $\lambda$ step size $\in \{10^{-3}, 3\times10^{-2},
3\times10^{-1}\}$ over $5$ seeds per configuration. Configurations were ranked
by steady-state adherence, leading candidates were re-run at $30$ seeds before
selection, and reward broke adherence ties. Canary, CPO and CVaR-CPO received
no search: all three inherit CPO's default trust region $\delta = 0.01$, and
PPO uses standard settings. The tuning budget therefore favors the
constrained baseline rather than the proposed method.

The sweeps also show that CPPO's failures in \cref{tab:safety} are not
tuning artifacts. On IcyLake the sweep is flat: all $12$ configurations
converge to $\hat{P} \in [0.32, 0.37]$ at $\hat{J} \approx 1$, and driving
$\lambda$ to its cap changes neither, so the CVaR penalty supplies no usable
adherence gradient at $\varepsilon = 0.02$ and CPPO behaves as unconstrained
PPO. On EcoAnt, adherence improves monotonically with gain up to $k_p = 600$,
bottoming out at roughly twice the threshold, then reverses as larger gains
($k_p \ge 1200$) destabilize the Lagrangian and destroy reward; no
configuration in the $40\times$ gain range reaches $\varepsilon$.

\subsection{Violation-Probability Jitter}
\label[appendix]{app:jitter}

Per seed, we partition the second half of training into 12 equal
non-overlapping blocks and compute block violation probabilities $p_b$
(exceedances over episodes) with episode counts $n_b$. A frozen policy would
still show block-to-block variance from finite sampling; we subtract this
binomial counting floor, $\mathrm{mean}_b[\bar{p}(1-\bar{p})/n_b]$ with
$\bar{p}$ the seed's pooled second-half rate, from the observed block variance
$\mathrm{Var}_b[p_b]$ and report the square root of the excess: the standard
deviation with which the converged policy's violation probability wanders.
Values are seed means with standard errors. The floor subtraction assumes
episodes are independent given the policy, which holds in these environments.


\end{document}